\documentclass{article}

\usepackage[preprint]{neurips_2025}

\usepackage[utf8]{inputenc} 
\usepackage[T1]{fontenc}    
\usepackage{hyperref}       
\usepackage{url}            
\usepackage{booktabs}       

\usepackage{amsmath}        
\usepackage{amsthm}

\usepackage{algorithm}
\usepackage{algorithm}
\usepackage{algpseudocode}  

\usepackage{amsfonts}       
\usepackage{nicefrac}       
\usepackage{microtype}      
\usepackage{xcolor}         

\usepackage{booktabs}
\usepackage{multirow}
\usepackage{array}
\usepackage{xcolor}

\usepackage{float}

\usepackage{geometry} 

\usepackage{natbib}

\usepackage{graphicx}

\usepackage[utf8]{inputenc} 
\usepackage[T1]{fontenc}    
\usepackage{hyperref}       
\usepackage{url}            
\usepackage{booktabs}       
\usepackage{amsfonts}       
\usepackage{nicefrac}       
\usepackage{microtype}      
\usepackage{xcolor}         

\title{Relative Overfitting and Accept-Reject Framework}

\author{
  Yanxin Liu$^{*}$ \\
  Yunnan University \\
  \texttt{12024213042@stu.ynu.edu.cn} \\
  \And
  Yunqi Zhang$^{*}$ \\
  Yunnan University \\
  \texttt{yunqizhang@ynu.edu.cn} \\
}

\begin{document}

\maketitle
\newtheorem{definition}{Definition}[subsubsection]
\newtheorem{theorem}{Theorem}[subsubsection]
\newtheorem{lemma}[theorem]{Lemma}
\newtheorem{assumption}[theorem]{Assumption}
\newtheorem{corollary}[theorem]{Corollary}

\begin{abstract}

The scaling of Large Language Models (LLMs) currently faces significant challenges. Model assembly is widely considered a promising solution to break through these performance bottlenecks. However, current ensembling methods are primarily guided by the statistical expectation that combining multiple models over large samples will lead to performance gains. We propose an ensemble framework that transitions from such stochastic, sample-dependent methods to a regular, controllable approach based on fine-grained model segmentation. This regularity governs how models are segmented to ensure performance improvement, how the magnitude of this improvement varies with model selection, and what factors determine its theoretical maximum. To formalize this pattern, we introduce the concept of'relative overfitting,' which is derived from the performance discrepancies between constituent models and builds a bridge between ensemble outcomes and the inherent attributes of these models. We detail the patterns of this framework within the domain of NLP and briefly describe its extensibility to other fields, such as computer vision (CV) and AI for science. Our approach was validated using both custom-built and pre-trained mainstream models across diverse benchmarks, including language modeling, long-context tasks, and question-answering (QA). The results indicate that the ensemble rules we proposed are generally effective and that we provide a rigorous proof of these rules in certain experimental scenarios. The proposed framework offers a new perspective for understanding ensemble theory and provides a systematic approach to addressing the performance bottlenecks of LLMs.

\end{abstract}

\section{Introduction}

The application of artificial intelligence, particularly large language models, has significantly improved people's productivity. For an extended period, scaling laws \cite{1, 2, 3}, driven by more data, greater computational power, and larger models, served as the guiding principle for training LLMs, allowing performance improvements through resource accumulation. However, with the rapid advancement of artificial intelligence, scaling laws face increasing scrutiny \cite{4, 5}, raising questions about whether they have reached a bottleneck. 

Ensemble learning has long been regarded as an effective means of overcoming the performance bottleneck of models. However, traditional ensemble learning theory relies on the improvement of performance expectations of multiple model combinations under large sample situations, which presents an irregularity; that is, we cannot ascertain whether the performance of an ensemble of two models will necessarily improve rather than decline, nor what the patterns of improvement are. 

In this work, we have constructed a framework that allows performance enhancement after ensemble learning to be transformed from being irregular and dependent on large samples into a systematic ensemble achieved through a fine-grained segmentation. This enables us to better understand the operational logic of ensemble methods to enhance their effectiveness in practical applications.

At the same time, we proposed the concept of relative overfitting. It posits that models with different structures exhibit fine-grained differences when modeling the target distribution. Within the same architecture, a model that better models the target distribution demonstrates relative overfitting compared to another model. This concept can be regarded as an extension of model diversity in ensemble learning, serving as a bridge between the performance of ensemble models and the inherent properties of the individual models involved.

Specifically, this paper proposes a universally applicable Accept-Reject (AR) framework and elaborates on it within the natural language processing (NLP) domain, primarily using language models as a medium. This choice is motivated by the abundance of scaling law-based model series in the field of NLP, which provides a robust experimental foundation. 

This framework introduces the concept of relative overfitting, linking the performance enhancement of ensemble models to the intrinsic properties of the models involved in the ensemble. Through a partition, we can maximize the favorable components of the ensemble while minimizing the unfavorable parts, thereby fully leveraging the advantages of each model within the ensemble. Moreover, this partitioning approach allows us to transform ensemble learning from being irregular and reliant on expectations under large samples to being systematic and controllable.

We have provided a detailed introduction to the AR law, which is based not only on our extensive experimental results across various model structures and datasets but also supported by rigorous proofs in specific scenarios. This facilitates a deeper understanding of our principles. The experimental results indicate that the principles we have proposed are universal, and by mastering the AR law, from an application perspective, this approach can achieve performance enhancements equivalent to or even superior to simply increasing the parameter count of LLMs in many cases, while significantly reducing parameter and computational costs. Additionally, from a theoretical standpoint, such regularity aids in our better understanding and control of this powerful tool, general learning.

\section{Related Work}
\label{gen_inst}

The extent of actual parameter requirements has been previously investigated. Theoretically, based on the Johnson-Lindenstrauss lemma \cite{13}, even a random projection, sampling a dimensionality reduction matrix from a Gaussian distribution, can compress N-dimensional features into $O(\log n / \epsilon^2)$ dimensions. Experimentally, the findings of ALBERT \cite{12} indicate that applying low-rank decomposition to BERT embeddings, reducing their dimensionality from 768 to 128, results in negligible performance degradation. These studies inspire us to propose a relative overfitting approach as a medium to communicate the underlying patterns between the original model and the ensemble model.

Before this work, considerable research had been conducted on model fusion. Mainstream approaches include traditional ensemble methods, such as the multi-model voting schemes commonly used in Kaggle competitions \cite{24}, which involve weighted combinations of outputs. Recent attempts within the LLM domain involve merging multiple models into a single entity \cite{25}. Furthermore, some earlier studies have explored the fusion of output levels for LLM \cite{26}. 

Unlike these studies, our research focuses on finding a bridge between the original model and the ensemble model through the concept of relative overfitting. Within the AR framework, our aim is to transform random ensembles into regular and controllable ensembles. We demonstrate that many methods can be categorized under this framework, allowing a better application of these methods through AR law.

\section{Relative Overfitting: Bridging Original and Ensemble Models}

This chapter will utilize SLMs and LLMs as exemplars exhibiting relative overfitting (assuming that LLMs generally achieve higher granularity in modeling the target distribution than SLMs within the same architecture) in the NLP domain to elaborate on the theoretical foundation of our method: the concept of relative overfitting.

According to the scaling law of LMs, models with more parameters generally demonstrate a greater capacity to represent details. This involves a trade-off between overfitting and underfitting. For models already proven excellent, both overfitting and underfitting are considered well-controlled. Therefore, our focus is on "relative overfitting and underfitting." A larger model possesses higher granularity in modeling the target distribution than a smaller one; an LLM exhibits relative overfitting compared to an SLM. This characteristic simultaneously drives its superior performance and creates limitations in other aspects. The underlying principle is the degree of fit to the target distribution. A model that fits the target distribution well inherently exhibits relative overfitting compared to a model that fits it poorly. Even with meticulous control, relative overfitting inevitably occurs compared to a lower-performing model. Consequently, this paper explores how to influence underfitting to influence overfitting, using LLMs and SLMs as the medium for discussion, given that LLMs typically relative overfitting concerning SLMs.

It is crucial to note that relative overfitting is not merely an extension of the concept of overfitting. Although both phenomena are due to model complexity, overfitting has associated control variables, and it can be managed by adjusting the training dataset without altering model parameters. However, relative overfitting lacks an external control variable, making it impossible to manage this phenomenon by manipulating another quantity. This is why we emphasize the endogenous nature of relative overfitting. Simply put, overfitting depends on the relationship between model complexity and dataset size, whereas relative overfitting primarily depends on the model's capacity to capture fine details of the target distribution.

In summary, relative overfitting enables one model to identify more semantic details than another. This means that LLMs can output sentences with substantial detail. Although generally beneficial, certain drawbacks that stem from this capability inspire our method. The issue arising from this higher semantic granularity is that although the output is more detailed, there is a greater likelihood of deviating from the core subject matter. LLMs strive to match the details of the target distribution as closely as possible, consequently increasing the probability of diverging from the main topic. Conversely, SLMs, constrained by their representational capacity, cannot inherently represent extensive detail and tend to concentrate on high-frequency words. In other words, an SLM might fail to represent adjectives, focusing instead on the high co-occurrence frequency of core components, whereas an LLM will pursue the most contextually appropriate and complete expression possible. Research indicates that in most scenarios, word distributions follow Zipf's law. Furthermore, high-frequency words preferentially fit within deep neural network frameworks \cite{53, 54}. The improvements offered by LLMs over smaller-parameter SLMs of the same architecture are primarily concentrated on low-frequency words; the fitting difference for extremely high-frequency words is relatively small between the two. Moreover, studies suggest that high-frequency words tend to be overestimated within commonly used frameworks, while inaccurately fitted low-frequency words are generally underestimated \cite{55, 56, 57, 58}.

We consider an LM $q$ and a Small Model (SM) $q'$.
For a specific high-frequency output class $h \in V$ (where $V$ is the set of all output classes), let $n = |S_h^q|$ and $m = |S_h^{q'}|$ denote the number of instances where the LM and SM predict class $h$, respectively.
Furthermore, let $P_L(h)$ and $P_S(h)$ represent the precision of LM and SM, respectively, when predicting class $h$ (i.e., $P(k_s=h \mid \text{model predicts } h)$).

Based on observations of model performance on high-frequency classes (detailed in Appendix \ref{B.1}), we propose the following assumption.

Assumption \ref{ass:1}: For a high-frequency output class $h$, we assume:
\begin{enumerate}
    \item \textit{Output Frequency Discrepancy}: The SM predicts class $h$ more frequently than the LM, that is, $n < m$.
    \item \textit{Limited LM Precision Advantage}: The LM precision advantage for class $h$ is bounded, specifically satisfying:
    \begin{equation}\label{eq:main_body_assumption1_bound}
    |P_L(h) - P_S(h)| < \frac{(m-n)P_S(h)}{n}
    \end{equation}
\end{enumerate}

This assumption directly leads to our first key theoretical result concerning the overall correct predictions for such mainstay classes:

Theorem \ref{thm:mainstay_deviation} (Mainstay Deviation): Under the assumption, for the high-frequency output class $h$, the total number of correct predictions by the SM, $N_S^h = m \cdot P_S(h)$, will be greater than that of the LM, $N_L^h = n \cdot P_L(h)$. This is formalized as:
\begin{equation}\label{eq:main_body_theorem1_inequality}
N_S^h > N_L^h \quad \text{or equivalently, } \quad m \cdot P_S(h) > n \cdot P_L(h)
\end{equation}
The Mainstay Deviation Theorem suggests that, for common categories where an SM exhibits a higher output propensity and the LM's precision edge is constrained, the SM may achieve a greater aggregate of correct identifications. This highlights a nuanced aspect of model scaling, where overall superiority does not always translate into superior performance in all submetric or output categories.

The complete formalization of all assumptions and detailed proofs are provided in Appendix \ref{A}.

Through this concept, we can better identify patterns to segment the ensemble of the original model in detail, allowing the performance of the ensemble model to reach its optimal potential.

\section{AR Framework}
\label{others}
We proposed an AR framework that quantifies the performance of the original model in a certain scenario through the concept of relative overfitting as a basis for segmentation, enabling the ensemble model to leverage the advantages of each model while minimizing the introduction of disadvantages. Without loss of generality, we will continue to use SLMs and LLMs as representative models for our discussion in the NLP domain. Section 4.1 will introduce the AR framework based on the simplest untrained output token-level ensemble. Sections 4.2, 4.3, and 4.4 specifically introduce our ensemble theoretical principles. All code is visible: https://anonymous.4open.science/r/130.

\subsection{Our Method}
As mentioned above, relative overfitting allows us to use the performance differences between the original models as a basis for segmentation, where we take the better-performing model (LLM) as the primary model and the other (SLM) as the secondary model. Based on the discussion of relative overfitting, we can assume that there will always be certain positions in SLM that exert a beneficial influence on LLM, and this influence can be simply quantified by the inherent properties of the models. The concept of segmentation is to ensure that the correct corrections from SLM to LLM are as widely accepted as possible, while erroneous changes are rejected as much as possible.

We will introduce a simple ensemble method as a medium for our theoretical discussion. The specific implementation is as follows: we allocate proportions based on certain inherent properties between SLM and LLM, thereby identifying the AR's segmentation point. Based on this segmentation point, we assign weights to the smaller and larger models and perform a weighted combination at the probability distribution layer of their outputs to obtain the final result. The weights assigned based on this segmentation point effectively establish a threshold to determine significant deviation. For the LLM's output, the SLM, holding minimal weight, generally has no impact, meaning that in most situations, we primarily utilize the LLM's detailed advantage and accept its judgment. However, once the conflict between the two reaches a critical point, the SLM strongly contradicts the LLM's judgment; we have reason to believe that the LLM's output progressively deviates from the core path. At this juncture, we consider rejecting the LLM's judgment. This enables us, in extreme situations, to take advantage of the core advantage of the SLM, thus achieving complementary advantages between models (as illustrated in \autoref{fig:my_first_figure}).

This methodology was selected to elucidate the AR framework due to the inherent property that both its outputs and weights reside entirely within the probability space. This characteristic offers an intuitive basis for discussing the AR framework. The key to the AR framework lies in using relative over-collaboration as a bridge to establish the relationship between the inherent properties of the original model, which serves as the basis for finding the segmentation points. Consequently, we have also extended our discussion of the AR framework by incorporating cosine similarity metrics, as detailed in the Appendix \ref{B.3}. Furthermore, based on the evidence presented in the Appendix \ref{A}, we believe that more methods embodying this concept can also be incorporated within this framework to achieve a more systematic ensemble.
\begin{figure}[t]
    \centering
    \includegraphics[width=12cm,height=14\baselineskip]{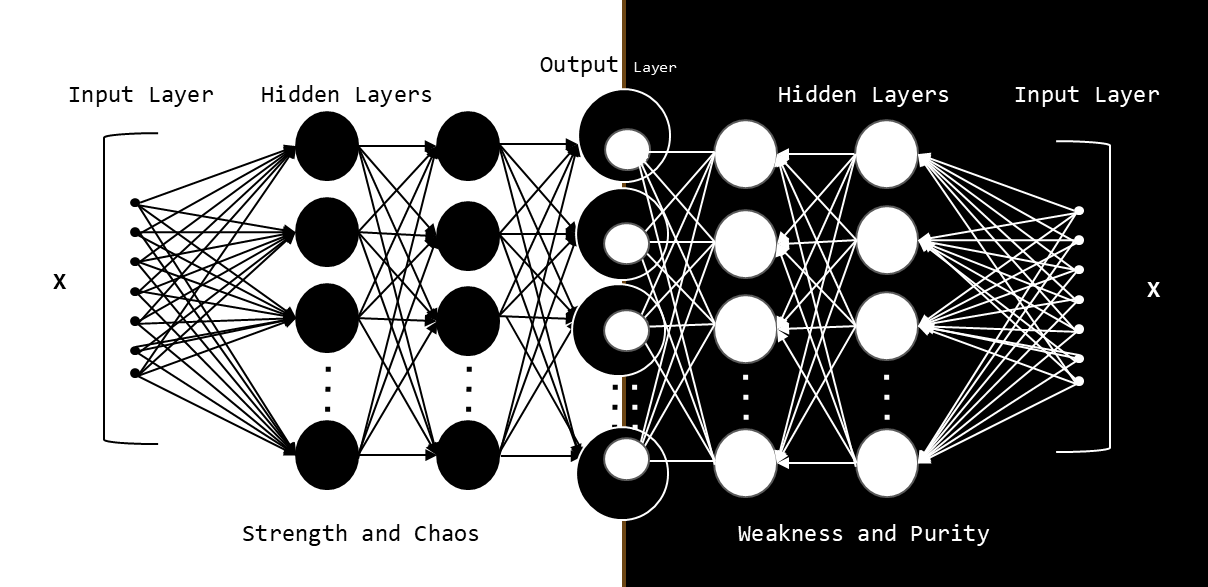} 
    \caption{A simple manifestation of an AR framework.}
    \label{fig:my_first_figure}
\end{figure}

Our segmentation framework allows us to separate adverse and beneficial influences at any fine-grained segmentation. In general, this means that even LMs with large performance disparities have their respective strengths and weaknesses. Our AR laws enable us to effectively leverage these advantages, no matter how minor they may seem. Specifically, applying softmax to the logit vectors $l_i$ ($i=1, 2$) output by the two models yields:
\begin{equation}
p_i = \mathrm{softmax}(l_i)
\end{equation}
Based on the threshold point of AR, an appropriate weight w (Weights corresponding to the LM) is determined, and a weighted sum is performed to obtain the final output:
\begin{equation}
p_{\text{ensemble}} = w \cdot p_1 + (1 - w) \cdot p_2
\end{equation}
It can be asserted that an SLM satisfying the threshold point of the AR condition will invariably exert a positive influence on the LLM. This influence is determined by the relative overfitting between the different models, is endogenous, and represents an inevitable outcome of the trade-offs involved. This explains its universal applicability, irrespective of model architecture or performance level. We discuss this point further in the subsequent experimental section, noting that while the magnitude of performance improvement may vary, an enhancement is consistently achievable.

Currently, this method only incurs linear complexity related to vocabulary size when a parallel implementation is considered. In many situations, compared to improving LLM performance by increasing the number of parameters, this training-free method offers significant savings in memory and computation time, particularly for long sequences based on the transformer architecture.

\subsection{AR Law}
The AR law constitutes the rule within the AR framework that defines the acceptance and rejection criteria for the judgments of various SLMs when operating with a fixed LLM. This chapter will illustrate this based on the ideas, discussions and experimental results of our approach, using the simple implementation of the AR framework described in Section 4.1 as an example.

First, we introduce the concept of the threshold point of AR. Let the parameter ratio between the LLM and the SLM be $k = \frac{|\theta_{\mathrm{LLM}}|}{|\theta_{\mathrm{SLM}}|}$, where $|\theta_{\mathrm{LLM}}|$ denotes the number of parameters in the LLM, and $|\theta_{\mathrm{SLM}}|$ denotes the number of parameters in the SLM. Since the performance of the same model can vary across different datasets, we use parameters as proxy indicators for the inherent attributes of the model. Although our discussion uses the parameter ratio as a metric, it is important to note that it represents the performance ratio and the underlying concept of relative overfitting. We define the weight ratio as $\beta = w/(1-w)$. Therefore, if the AR weight allocation ratio that achieves the optimal performance improvement is $\alpha$ ($\alpha$ is designed to maximize the performance improvement of the ensemble model $\beta$), then we have:
\begin{equation}
\alpha \propto k
\end{equation}
Second, we introduce the concept of the improvement domain, denoted as $[D, \infty]$, where $D < \alpha$. For any weight allocation ratio $\beta$, the model can improve performance when $\beta \in [D, \infty]$. Furthermore, when $\beta < \alpha$, the magnitude of performance improvement increases, and when $\beta > \alpha$, the magnitude of performance improvement decreases. When $\beta = \infty$, it approximates a single model. We also have:
\begin{equation}
R \propto k
\end{equation}
When $k$ is extremely large, meaning the performance (parameter) gap between the models is vast, $R$ approaches infinity ($R \to \infty$). In this scenario, we consider the performance improvement to be near zero. Of course, this situation rarely occurs. To verify this extreme case, we designed an experiment using a model with abysmal performance and very high perplexity to construct the AR framework. In this case, the performance improvement was almost negligible.

Finally, the magnitude of performance improvement is inversely proportional to the weight ratio $\beta$.

The points described above constitute the basic AR law, which is defined on the basis of the theoretical ideas presented earlier. Next, we will discuss the selection of the AR threshold point.

\subsection{The Threshold Point of AR}
Based on the ideas discussed in the previous section, we examine the threshold point of the optimal allocation ratio $\alpha$ based on the experimental results. 

We decompose the optimal ratio $\alpha$ into two conflicting components. The first is an ascending component. If the parameter ratio $k$ between the LLM and the SLM is large, we should use a larger $\alpha$ to better utilize the LLM's capability for detail, essentially its conventionally understood performance. According to Kaplan \cite{1}, LLM performance follows a power law concerning parameter count, suggesting($p$ is the power-law exponent):
\begin{equation}
\alpha \propto k^p
\end{equation}
The second component characterizes the magnitude at which the AR framework operates effectively. In reality, $k$ only describes the gap between LLM and SLM but does not capture the starting point of this gap. We introduce the parameter $\theta_{LLM}$ (the LLM's parameter count) to represent the starting point where this influence takes effect. For example, the ratio between parameters 14B, 7B, 14M, and 28M is also 2, but the magnitude of the differences is entirely different. Relative overfitting cases arising from different magnitudes will also vary. Therefore, as noted in Section 4.2, we discuss how to influence it based on the AR framework and law under a fixed LLM. The average optimal weight $\alpha$ of each model in various datasets in all experiments from Section 5.1 is presented in Appendix \ref{B.2}.

\subsection{Improvement Magnitude}

In this subsection, we discuss the magnitude of performance improvement contributed by the lower performing model to the higher performing model within the AR framework. Unlike previous discussions, although dependent on weights, the magnitude of performance improvement is influenced by more complex factors in practice than the relatively strict proportional and inverse relationships explored in Section 4.2. Based on the experimental results, we identify two empirical regularities. First, as the weight ratio (associated with the LLM) increases, the proportion and influence of the SLM decrease. However, this inverse relationship is not absolute; exceptions were observed in experiments. This is because the change in proportion is an integral part of the AR law itself: reducing the SLM's proportion serves the mechanism's requirements and is not merely about decreasing the SLM's influence. we see that the actual optimal weight ratio $\alpha$ increases only with $k$ (the parameter ratio). Secondly, as stated in Section 4.3. We observe that while the actual optimal weight ratio $\alpha$ increases with our experiments, it does not demonstrably increase with the scale of the main model parameters $\theta_{LLM}$. This provides evidence for the stability of the AR framework: that is, as long as an appropriate k is maintained, the increase in the model's performance remains stable with the growth of $\theta_{LLM}$, and in some scenarios, there is even a rising trend (as seen in Section 5.1). This means that, according to the AR law, we need not worry about a significant reduction in performance due to extensive main model parameters. This finding supports our belief that this approach can positively impact the bottleneck observed in scaling laws. Moreover, under the scaling law, as the LLM parameters increase, the rate of improvement decreases rapidly.

\subsection{Proof of part of the AR law}
Chapters 4.2, 4.3, and 4.4 present empirical conclusions derived from our tens of thousands of experimental results. In fact, we provide a rigorous proof concerning the AR law based on the ACC metric, under specific assumptions, as seen in Appendix \ref{A}. This proof not only demonstrates a part of the correctness of AR law, but its proof process also more comprehensively illustrates the ideas contained within the AR framework.

In particular, we can naturally extend the proof of similar laws to multi-model ensemble scenarios for wider applications. Taking the dual model scenario as an example is merely for the convenience of demonstrating the properties and experiments.

\section{Experiments}
\label{others}

This article chooses the splitting point based on AR law and conducts experimental verification using the introduced simple ensemble method. In the field of NLP, we designed experiments considering existing mainstream scaling-law-based LMs and custom-built LMs based on prevalent architectures. The aim is to verify the universality of our approach. First, we tested our method on several existing scaling law-based model series (e.g., GPT-2 \cite{33}). Second, we performed ablation studies using a model built on the Transformer \cite{31} architecture, deliberately avoiding task-specific design modifications and comparing with existing models. Finally, we retrained models of varying scales based on the mainstream architectures LSTM \cite{29, 30}, Transformer, and GRU \cite{32} for further experimentation. This chapter demonstrates that the proposed framework is universally applicable to various models, including mainstream models, custom models, and combinations thereof. Furthermore, based on the principle of relative overfitting, we validated the AR framework in computer vision (CV) and AI for science.

\subsection{Mainstream Language Model Testing}

This paper conducted experiments using the proposed AR framework on several pre-trained models. The tests were performed using GPT-2 \cite{33} on fundamental language modeling benchmarks including WikiText2, WikiText103 \cite{34}, text8 \cite{14}, Pile \cite{39, 40} (using a randomly sampled subset of the Pile test set), and 1BW \cite{38}, as well as on the long-context benchmark LAMBADA \cite{35} and the QA benchmarks CBT \cite{36} and the subject examination benchmark ARC \cite{37}. We also evaluated the performance of models with different architectures: Pythia \cite{41} (multiscale scaling law), RWKV \cite{42} (combining RNN and Transformer), Mamba \cite{43} (state-space model series), and Qwen \cite{44, 45} (commercial LLM) - using the AR framework for basic language modeling tasks. We used the LAMBADA dataset for the language modeling evaluation (perplexity) and its original long-context dependency evaluation to compare the behavior under the AR framework between general language modeling and specific tasks. 

The results in \autoref{tab1} demonstrate that all models achieve universal performance improvements through our AR framework. We can achieve superior performance in many scenarios with significantly fewer parameters and less computation time. For instance, on LAMBADA (PPL), Pile, and other benchmarks, our method consistently outperforms the strategy of simply increasing the LLM's parameter count, achieving superior parameters, computation time, and performance. This advantage becomes more pronounced as the model parameters increase. For example, based on the improvement rate observed when scaling GPT-2 from 762M to 1.5B, and roughly estimating the improvement from 1.5B to 3B using the power law model proposed in scaling laws, the performance of the 1.5B+762M combination almost universally surpasses the projected 3B model across all benchmarks. Notably, our method's advantages become more significant with larger models, while the linear complexity related to vocabulary size becomes negligible compared to increasing the parameter count. Furthermore, as the main model's parameters grow, the magnitude of performance improvement increases rather than diminishes, making us optimistic about its performance on even larger models and hopeful that it can contribute to alleviating the scaling law bottleneck.
\begin{table}
\caption{The table shows the test results of the classic model GPT2 based on the AR framework with four parameter sizes from 117M to 1.5B. LTD is the long context benchmark under the LAMBADA dataset. The results of RWKV, Mamba, Pythia, and Qwen can be found in the Appendix \ref{D.1}.}
\label{tab1}
\centering
\resizebox{\textwidth}{!}{
\begin{tabular}{llcccccccccc}
  \toprule
  LLM & SLM & WikiText2 & WikiText103 & text8 & 1BW & Pile(part) & LAMBADA & LAMBADA & CBT-CN & CBT-NE & ARC \\
  GPT2 & GPT2 & PPL $\downarrow$ & PPL $\downarrow$ & BPC $\downarrow$ & PPL $\downarrow$ & PLL $\downarrow$ & PPL $\downarrow$ & PLL(LTD) $\downarrow$ & ACC $\uparrow$ & ACC $\uparrow$ & ACC $\uparrow$ \\
  \midrule
  \multicolumn{2}{c}{117m} & 23.52 & 29.16 & 1.21 & 63.97 & 18.66 & 43.08 & 24.99 & 78.84 & 72.40 & 22.35 \\
  \multicolumn{2}{c}{345m} & 17.99 & 21.08 & 1.10 & 52.75 & 13.74 & 33.66 & 12.66 & 82.88 & 75.36 & 25.00 \\
  \multicolumn{2}{c}{762m} & 15.84 & 18.15 & 1.05 & 45.67 & 13.20 & 30.61 & 9.33 & 84.68 & 76.32 & 24.40 \\
  \multicolumn{2}{c}{1.5b} & 14.80 & 16.51 & 1.02 & 43.81 & 12.15 & 29.96 & 7.80 & 86.92 & 78.92 & 26.62 \\
  \midrule
  \multicolumn{12}{c}{\textbf{Improvement of the evaluation metric relative to the best performing single model (\%)}} \\
  \midrule
  345m & 117m & 1.15\% & 0.62\% & 0.32\% & 2.16\% & 1.11\% & 1.63\% & 0.28\% & 0.24\% & 2.34\% & 1.71\% \\
  \midrule
  762m & 117m & 0.77\% & 0.42\% & 0.20\% & 1.15\% & 4.68\% & 1.81\% & 0.01\% & 0.43\% & 2.46\% & 1.40\% \\
  762m & 345m & 2.73\% & 2.17\% & 1.00\% & 2.26\% & 9.51\% & 4.13\% & 1.07\% & 0.85\% & 1.52\% & 3.41\% \\
  \midrule
  1.5b & 117m & 0.82\% & 0.50\% & 0.27\% & 1.52\% & 4.52\% & 2.36\% & 0.12\% & 0.37\% & 1.57\% & 1.28\% \\
  1.5b & 345m & 2.01\% & 1.47\% & 0.72\% & 2.16\% & 8.41\% & 4.92\% & 0.54\% & 0.18\% & 0.66\% & 2.88\% \\
  1.5b & 762m & 3.85\% & 3.24\% & 1.38\% & 4.72\% & 4.97\% & 6.99\% & 3.31\% & 0.18\% & 0.61\% & 1.92\% \\
  \bottomrule
\end{tabular}
}
\end{table}
\vspace{-6pt}

\subsection{Ablation Experiments with Custom-Built Models Based on Mainstream Architectures}

In this section, we use the vocabularies corresponding to Pythia, RWKV, Mamba, and GPT-2 to verify the universality and flexibility of our proposed method. Without incorporating any specific design elements, we employed a simple Transformer architecture to train a reference model with only approximately 180M parameters. In practice, the standalone performance of this model is not particularly strong.

We better demonstrate the regularities of ensemble theory under the AR framework through the design of experiments. In WikiText103 and LAMBADA datasets, we performed predictions using only the model trained on WikiText2, this is to illustrate that fine-grained segmentation under the AR framework allows even poorly performing models to have a positive influence on better models rather than a negative one. This was done to confirm that the performance improvement achieved by the AR framework is not merely due to having "seen the data". Furthermore, to verify the effectiveness of our method, we trained models on the 1BW and text8 datasets using their respective training sets to achieve substantial performance gains. Training different models (with poor, moderate, and good performance) based on different datasets aims to better observe the relative overfitting and performance of the AR framework under various circumstances. In addition, we will present the results based on the same model in Appendix D.2.

The results in Table 2 indicate that, through appropriate segmentation, even subpar models can still exert a positive influence. Average models are capable of significantly enhancing the performance of the primary model, while the positive impact generated by superior models is quite remarkable.

\begin{table}
\caption{All models used the simplest Transformer architecture. For WikiText2, WikiText103, and LAMBADA results, models were trained on the WikiText2 training set, while for 1BW and text8 results, models were trained on the 1BW and text8 training sets, respectively.}
\label{tab2}
\centering
\begin{tabular}{lccccc}
  \toprule
  \multicolumn{1}{c}{Model} & WikiText2 & WikiText103 & LAMBADA & text8 & 1BW \\
  \multicolumn{1}{c}{Trained Model} & PPL $\downarrow$ & PPL $\downarrow$ & PPL $\downarrow$ & BPC $\downarrow$ & PPL $\downarrow$ \\
  \midrule
  GPT2-Train & 91 & 322 & 9084 & 1.03 & 131 \\
  Other-Train & 94 & 314 & 12084 & 1.10 & 45 \\
  \midrule
  \multicolumn{6}{c}{\textbf{Improvement of the evaluation metric relative to the best performing single model (\%)}} \\
  \midrule
  GPT2-1.5b & 9.46\% & 6.33\% & 1.02\% & 26.12\% & 22.40\% \\
  Pythia-2.8b & 3.76\% & 3.09\% & 0.10\% & 19.61\% & 47.74\% \\
  Mamba-2.8b & 4.26\% & 3.56\% & 0.10\% & 18.68\% & 45.71\% \\
  RWKV-1.5b & 4.64\% & 3.67\% & 0.08\% & 20.53\% & 49.02\% \\
  \bottomrule
\end{tabular}
\end{table}

\begin{figure}[t]
    \centering
    \includegraphics[width=\textwidth]{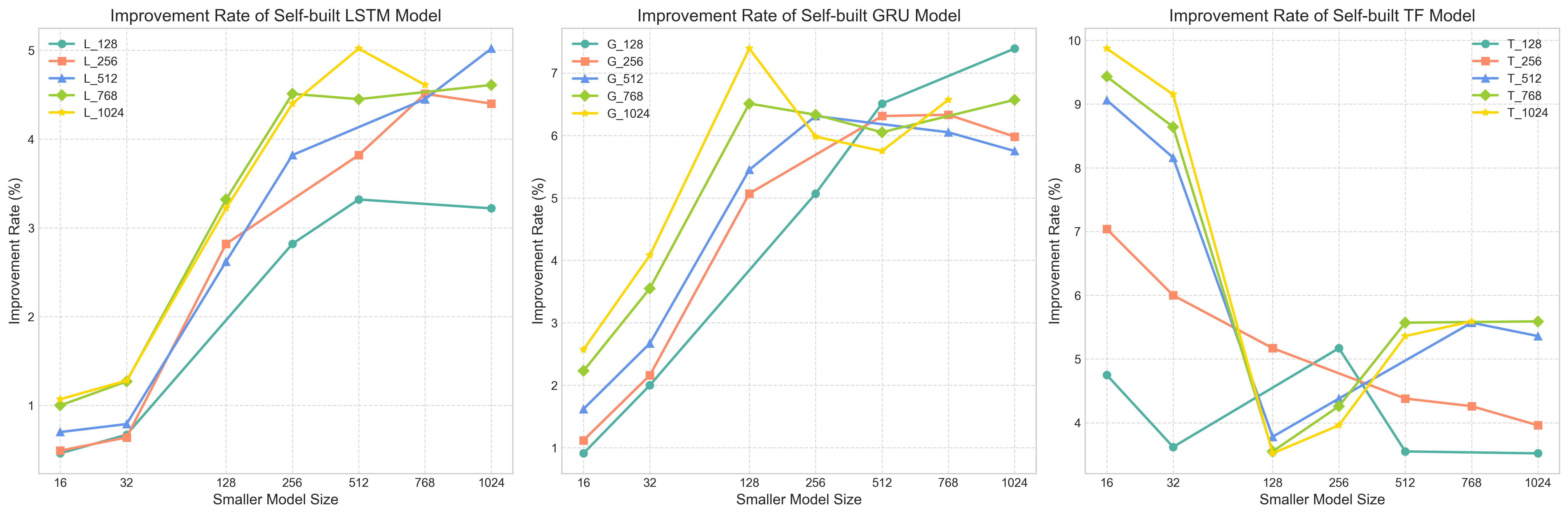}
    \vspace{-0.5cm} 
    \caption{Test results of AR framework based on simple LSTM, Transformer and GRU architectures under different embedding dimensions of SLM and LLM.}
    \label{p2}
\end{figure}
\vspace{-6pt}

\subsection{Self-built model testing based on mainstream architectures}

This paper conducted tests based on the current mainstream model architectures LSTM \cite{29, 30}, Transformer \cite{31}, and GRU \cite{32}, utilizing the same vocabulary but varying embedding dimensions. In practice, we used very simple architectures \cite{49, 50, 51, 52} without any specific modifications to the model design. The ensemble theory we proposed is still valid (as illustrated in \autoref{p2}).

\subsection{Experiments in the rest of the machine learning field}

Based on the principles of relative overfitting and the AR framework, we employed a similar methodology to conduct experiments with our approach across a broader range of machine learning domains, including CV and AI for Science, thus validating the general effectiveness of our proposed framework. We conducted experiments using different sizes of ResNet\cite{64} models and ESM\cite{65} models on image classification and protein sequence prediction tasks (in the Appendix \ref{C}).

\subsection{AR law}
In order to illustrate the universality of the ensemble theory we have proposed, this section presents a portion of the results from over ten thousand experiments (as illustrated in \autoref{p3}), while another portion is allocated for more detailed weight searches. The remaining results of the organized section can be seen in the appendix.
\begin{figure}[t]
    \centering
    \includegraphics[width=\textwidth]{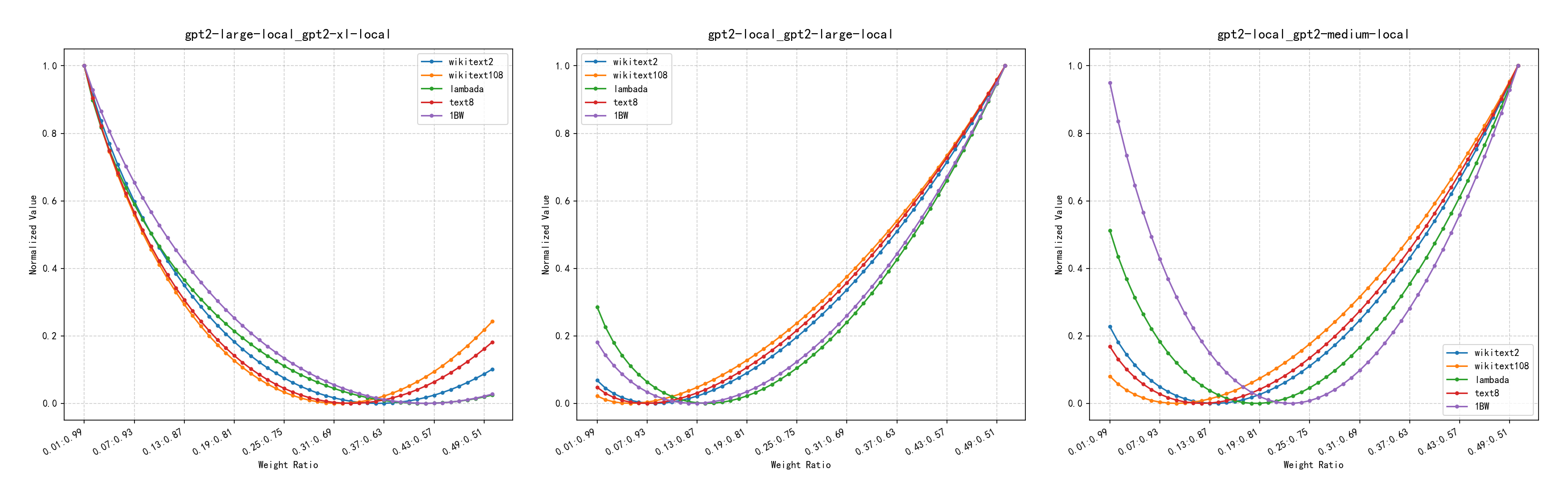}
    \vspace{-0.5cm} 
    \caption{AR law in the GPT2 model combination.}
    \label{p3}
\end{figure}

\section{Discussion and Future work}

We discuss limitations and future directions.

\begin{itemize}
\item \textbf{Model Selection Trade-offs:} For homogeneous scaling models, while performance showed stability (no significant fluctuations) as the main model grew, the magnitude of improvement decreased as the parameter ratio expanded. Larger ratios required more extreme leverage, reducing the effectiveness of SLM and potentially limiting resource savings. In specific domains, nonhomogeneous models are a better choice.
\item \textbf{Better Improvement Magnitude:} Although this paper proposes the AR law based on relative overfitting and demonstrates its universal, stable, and practical improvements along with good flexibility under a suitable threshold point, significant room for progress remains beyond the decent performance improvements already achieved. To reveal fundamental principles, our method achieved these results despite lacking any specific model design, highlighting the potential yet to be unlocked. Therefore, with a deeper understanding of the entire system and relative overfitting, key research questions worth exploring include different combination methods based on the AR law, the selection and innovation of low-parameter reference model architectures serving as backbones, and how to maintain the information purity of the SLM better. We anticipate that addressing these questions will lead to greater magnitudes of improvement.
\item \textbf{Limitations of the AR law:} As noted in this paper, the AR framework involves complex relative overfitting. Although we have proposed fundamental principles applicable in most scenarios, the inherent black-box nature and uncertainty of neural networks mean that minor fluctuations can occur within the general direction established by the AR Law. Consequently, we believe that the complex influence mechanisms of relative overfitting warrant further in-depth research.
\item \textbf{Broader Applicability:} Although this paper focuses mainly on the NLP domain and has also conducted some experiments in CV and AI for science, a broader discussion in more fields was not feasible due to limited time and resources. The applicability of the relative overfitting phenomenon and the proposed AR framework in a broader range of domains warrants further exploration.

\end{itemize}

\clearpage
\newpage
{\small
\nocite{*}
\bibliographystyle{unsrt} 
\bibliography{references}
}

\clearpage
\newpage
\appendix
\renewcommand{\thesubsection}{\Alph{subsection}} 
\section*{Appendix}
\subsection{Proof of Theorem}
\label{A}
\subsubsection{Definitions}
\label{A.1}
\begin{definition}
Let $q$ and $q'$ be two models sharing the same input and output spaces. If model $q$ has a larger number of parameters (or higher model complexity) than model $q'$, we define $q$ as the Large Model (LM) and $q'$ as the Small Model (SM).
\end{definition}

\begin{definition}
Let $S$ be the set of all samples (or positions) to be predicted in a given task, its size being $N = |S|$. Let $s \in S$ represent a specific sample. Let $V = \{1, 2, 3, \dots, |V|\}$ be the set of indices for all possible output classes (dimensions), where the output dimensions of the two models correspond one-to-one.
\end{definition}

\begin{definition}
Let $k_s \in V$ be the index of the true class for sample $s$. Let $q_s$ be the predicted output vector of model $q$ for sample $s$.
\end{definition}

\begin{definition}
Let $q_{s,i}$ be the probability assigned by model $q$ to sample $s$ belonging to class $i$. Similarly, $q'_{s,i}$ is defined for model $q'$.
\end{definition}

\begin{definition}
Model $q$ predicts sample $s$ correctly if and only if $\arg\max_{i \in V}(q_{s,i}) = k_s$.
\end{definition}

\begin{definition}
The entire set of samples in the task is partitioned into three sets: $T$, $F$, and $N$.
\begin{itemize}
    \item The set $T$ consists of samples where the LM predicts incorrectly, but the SM predicts correctly:
    $$T = \{s \in S \mid \arg\max_{i \in V}(q_{s,i}) \neq k_s \text{ and } \arg\max_{i \in V}(q'_{s,i}) = k_s\}$$
    \item The set $F$ consists of samples where the LM predicts correctly, but the SM predicts incorrectly:
    $$F = \{s \in S \mid \arg\max_{i \in V}(q_{s,i}) = k_s \text{ and } \arg\max_{i \in V}(q'_{s,i}) \neq k_s\}$$
    \item The set $N$ consists of samples where both LM and SM predict correctly, or both predict incorrectly.
\end{itemize}
We have $T \cup F \cup N = S$, and $T, F, N$ are mutually disjoint.
\end{definition}

\begin{definition}
The accuracy (ACC) of a model on a prediction task $S$ is defined as (where $I(\cdot)$ is the indicator function):
$$ACC_q = \frac{1}{|S|} \sum_{s \in S} I(\arg\max_{i \in V}(q_{s,i}) = k_s)$$
\end{definition}

\subsubsection{Assumptions and Core Theorems}
\label{A.2}
\begin{assumption}\label{ass:1}
The improvements of larger models are primarily concentrated in the tail of the distribution. Assume there exists a high-frequency output class $h \in V$. Let $S_h^q = \{s \in S \mid \arg\max_i q_{s,i} = h\}$ be the set of samples predicted as $h$ by LM, with size $n = |S_h^q|$. Let $S_h^{q'} = \{s \in S \mid \arg\max_i q'_{s,i} = h\}$ be the set of samples predicted as $h$ by SM, with size $m = |S_h^{q'}|$.

We assume the following conditions hold\cite{53,54,55, 56, 57, 58}:
\begin{enumerate}
    \item The SM outputs class $h$ more frequently than the LM, i.e., $n < m$.
    \item The advantage of LM's prediction accuracy for class $h$, $P_L(h) = P(k_s = h \mid s \in S_h^q)$, over SM's accuracy, $P_S(h) = P(k_s = h \mid s \in S_h^{q'})$, is limited. Specifically, this advantage satisfies the relation:
    $$|P_L(h) - P_S(h)| < \frac{(m-n)P_S(h)}{n}$$
\end{enumerate}
\end{assumption}

\begin{theorem}[Mainstay Deviation]\label{thm:mainstay_deviation}
Under Assumption \ref{ass:1}, for class $h$, the total number of correct predictions by SM is greater than that by LM.
\end{theorem}
\begin{proof}
The total number of correct predictions for class $h$ by LM is $N_L^h = n \cdot P_L(h)$.
The total number of correct predictions for class $h$ by SM is $N_S^h = m \cdot P_S(h)$.

We want to prove $N_L^h < N_S^h$, which is equivalent to proving $m \cdot P_S(h) - n \cdot P_L(h) > 0$.

We transform the difference:
\begin{align*}
m \cdot P_S(h) - n \cdot P_L(h) &= m \cdot P_S(h) - n \cdot P_S(h) + n \cdot P_S(h) - n \cdot P_L(h) \\
&= (m-n)P_S(h) - n(P_L(h) - P_S(h))
\end{align*}

From Assumption \ref{ass:1}, we have:
$$(P_L(h) - P_S(h)) \le |P_L(h) - P_S(h)| < \frac{(m-n)P_S(h)}{n}$$

Thus,
$$(m-n)P_S(h) - n(P_L(h) - P_S(h)) > (m-n)P_S(h) - n\left(\frac{(m-n)P_S(h)}{n}\right) = 0$$
The inequality is proven.

It can be easily shown that Theorem \ref{thm:mainstay_deviation} holds for any high-frequency output dimension satisfying Assumption \ref{ass:1}.
\end{proof}

\begin{lemma}[Exchange Condition]\label{lemma:exchange_condition}
Let the fused model be $q_{new,s,i} = w \cdot q_{s,i} + (1-w) \cdot q'_{s,i}$ for $w \in (0,1)$. On sample $s$, if the LM's original prediction is $i$ (i.e., $\arg\max(q_s) = i$), for the fused model's prediction for class $j$ to have a higher probability than for class $i$, it must satisfy:
$$w \cdot q_{s,j} + (1-w) \cdot q'_{s,j} > w \cdot q_{s,i} + (1-w) \cdot q'_{s,i}$$

For $w \in (0,1)$, this inequality is equivalent to:
$$\frac{q_{s,i} - q_{s,j}}{q'_{s,j} - q'_{s,i}} < \frac{1-w}{w}$$

Clearly, we need SM's prediction for $j$ to be greater than its prediction for $i$:
$$q'_{s,j} - q'_{s,i} > 0$$
This is termed the \textbf{exchange condition}.

Furthermore, if $\arg\max(q_{new,s}) = j$, then the fused model's prediction changes from $i$ to $j$. This is termed the \textbf{strict exchange condition} if the exchange condition is also met.

We define the left side as the Exchange Threshold, denoted $ET(s, i \to j)$. A flip occurs only when $ET < \frac{1-w}{w}$.

This lemma implies that only when the exchange condition is met can the fused model's prediction potentially differ from the LM's prediction. Any case not satisfying this condition cannot change the prediction outcome. Moreover, only when the strict exchange condition is met can the change in prediction outcome be determined.
\end{lemma}

\begin{assumption}[Stratification of Predictions]\label{ass:2}
Typically, a model exhibits greater discrimination for samples it predicts correctly\cite{68,69,70}. Specifically, we assume situations exist where: for some samples correctly predicted by model $q$ (e.g., $b \in F$), the difference between the probability of the correctly predicted class $q_{b,k_b}$ and the probability of the second most likely class (or any other incorrect class $x_b$, $q_{b,x_b}$), i.e., $(q_{b,k_b} - q_{b,x_b})$, is greater than the difference for some samples incorrectly predicted by model $q$ (e.g., $a \in T$) between its incorrectly predicted class probability $q_{a,i_a}$ and the true class probability $q_{a,k_a}$ (or some other reference class $z_a$, $q_{a,z_a}$), i.e., $(q_{a,i_a} - q_{a,z_a})$.

Formally, there exists at least one sample $a \in T$ such that for all samples $b \in F$ (if $F$ is non-empty):
$$(q_{b,k_b} - \max_{j \neq k_b} q_{b,j}) > (q_{a,i_a} - q_{a,k_a})$$
where $k_b = \arg\max(q_b)$ is LM's correct prediction for $b$, $i_a = \arg\max(q_a)$ is LM's incorrect prediction for $a$, and $k_a$ is the true class of $a$.

A similar condition holds for model $q'$:
$$(q'_{b,k_b} - \max_{j \neq k_b} q'_{b,j}) > (q'_{a,i'_a} - q'_{a,k_a})$$
Let $A$ be the set of all samples $a \in T$ satisfying this condition.
\end{assumption}

\begin{theorem}[AR Law: Improvement, Not Degradation]\label{thm:ar_law_improvement}
We say Condition (1) holds if there exists a position $s \in S$ where SM's predictive capability is stronger than LM's. Specifically, LM predicts incorrectly, $k_s \neq \arg\max(q_s)$, and SM's predicted probability for the true class $k_s$ is greater than its predicted probability for LM's incorrect prediction $i_s = \arg\max(q_s)$, i.e., $q'_{s,k_s} > q'_{s,i_s}$. (From Theorem \ref{thm:mainstay_deviation} and experimental results, this condition is generally met.)

Under Assumption \ref{ass:2}:
For a linear fusion model $q_{new} = w \cdot q + (1-w) \cdot q'$ (where $w \in (0,1)$) of a large model $q$ and a small model $q'$, if and only if Condition (1) holds, a weight $w$ can always be found such that the accuracy of the fused model $ACC_{new}$ is potentially higher than, and not worse than, the accuracy of model $q$, i.e., $ACC_{new} \ge ACC_q$.

If, additionally, for the set $A$, there exists at least one position satisfying the strict exchange condition, then the accuracy of the fused model $ACC_{new}$ is higher than the accuracy of model $q$, i.e., $ACC_{new} > ACC_q$. In fact, for every predicted position $s$, the conclusion is that the accuracy will not decrease.
\end{theorem}
\begin{proof}
From Lemma \ref{lemma:exchange_condition}, for the weighted output to potentially differ, the exchange condition must be met. For $\arg\max(q_s) = i$:
$$w \cdot q_{s,j} + (1-w) \cdot q'_{s,j} > w \cdot q_{s,i} + (1-w) \cdot q'_{s,i}$$
For output $j$ to replace output $i$ as the final output, it is also required that:
$$\arg\max(w \cdot q_s + (1-w) \cdot q'_s) = j$$

Assume Condition (1) does not hold. Then for any position $s \in S$:
If $q_{s,i} = \arg\max(q_s) = k_s$, then $q$ is already optimal at this position, and fusion can only lead to a worse result.
If $q_{s,i} = \arg\max(q_s) \neq k_s$, but $q'_{s,k_s} < q'_{s,i_s}$ (where $i_s$ is the LM's incorrect prediction), then
$$w \cdot q_{s,k_s} + (1-w) \cdot q'_{s,k_s} < w \cdot q_{s,i_s} + (1-w) \cdot q'_{s,i_s}$$
This means the true class $k_s$ can never replace the LM's incorrect prediction $i_s$ as the final output. Thus, the necessity of Condition (1) is proven.

Consider the partition of the task into sets $T, F, N$. Let $t \in T$ and $f \in F$.
For set $N$: if both models predict correctly, the weighted result cannot change the original $\arg\max$. If both predict incorrectly, the result cannot worsen. Thus, for set $N$, weighting can only potentially improve ACC.

For set $T$: if the exchange condition $\frac{q_{t,i_t} - q_{t,k_t}}{q'_{t,k_t} - q'_{t,i_t}} < \frac{1-w}{w}$ (where $i_t$ is LM's incorrect prediction and $k_t$ is the true class) is met, the incorrect result $i_t$ can be replaced by the correct result $k_t$.

For set $F$: if the exchange condition $\frac{q_{f,k_f} - q_{f,j_f}}{q'_{f,j_f} - q'_{f,k_f}} < \frac{1-w}{w}$ (where $k_f$ is LM's correct prediction and $j_f$ is SM's incorrect prediction) is met, the true result $k_f$ can be replaced by an incorrect result $j_f$, causing ACC to decrease.

By Assumption \ref{ass:2}, there exists some $a \in A \subseteq T$ such that for all $f \in F$:
$$(q_{f,k_f} - \max_{j \neq k_f} q_{f,j}) > (q_{a,i_a} - q_{a,k_a})$$
and
$$(q'_{a,k_a} - q'_{a,i'_a}) > (q'_{f,j_f} - q'_{f,k_f})$$
This implies that the exchange threshold for correcting $a$ is smaller than the exchange threshold for incorrectly changing $f$:
$$\frac{q_{a,i_a} - q_{a,k_a}}{q'_{a,k_a} - q'_{a,i_a}} < \frac{q_{f,k_f} - q_{f,j_f}}{q'_{f,j_f} - q'_{f,k_f}}$$

At this point, we can always find a $w$ such that:
$$\frac{q_{a,i_a} - q_{a,k_a}}{q'_{a,k_a} - q'_{a,i_a}} < \frac{1-w}{w} < \frac{q_{f,k_f} - q_{f,j_f}}{q'_{f,j_f} - q'_{f,k_f}}$$
Since $a \in T$, this means that a potential correct substitution in set $T$ is satisfied, while an incorrect substitution in set $F$ is not. Thus, ACC may increase and cannot decrease.

If there is at least one point $a' \in A$ that satisfies the strict exchange condition, then at this position, the incorrect result will be replaced by the correct result, causing an increase in ACC. There are N additional episodes that will not cause a decline in accuracy, thus there exists a w such that at any prediction position s, accuracy will not decrease.
\end{proof}

\begin{theorem}[AR Law: Maximum Improvement Magnitude]\label{thm:ar_law_max_improvement}
Under the conditions of Theorem \ref{thm:ar_law_improvement}, let $A$ be the set of all positions satisfying Assumption \ref{ass:2}. Let $R$ be the set of all positions satisfying the strict exchange condition, where $R \subseteq A \subseteq T \subseteq S$. The accuracy improvement $\Delta ACC = ACC_{new} - ACC_q$ of the fused model $q_{new}$ relative to the original large model $q$ has the following properties:
\begin{itemize}
    \item $\Delta ACC$ is a strictly monotonically increasing function of the size of set $R$, $|R|$.
    \item $\Delta ACC$ is a monotonically increasing function of the sizes of sets $A$ and $T$, $|A|$ and $|T|$.
\end{itemize}
\end{theorem}
\begin{proof}
According to Theorem \ref{thm:ar_law_improvement}, by choosing an optimal fusion weight $w$, we can ensure that only samples in set $T$ are corrected, and samples in set $F$ are not incorrectly changed.

By definition, $R$ is the set of all samples that can be safely corrected within the optimal weight range. Then, the maximum accuracy improvement $\Delta ACC_{max}$ directly depends on the size of $R$:
$$\Delta ACC_{max} = \frac{|R|}{|S|}$$
From this equation, it is clear that if $|R|$ increases, $\Delta ACC_{max}$ will also increase proportionally. Therefore, $\Delta ACC_{max}$ is a strictly monotonically increasing function of $|R|$.

By definition, we have the subset relations $R \subseteq A \subseteq T$. This directly leads to the inequality relations for set sizes: $|R| \le |A| \le |T|$.
Therefore, the maximum accuracy improvement is limited by $|A|$ and $|T|$:
$$\Delta ACC_{max} = \frac{|R|}{|S|} \le \frac{|A|}{|S|} \le \frac{|T|}{|S|}$$
This inequality shows that the upper limit of ACC improvement is determined by $|A|$ and $|T|$.

Now consider monotonicity. If set $A_1 \subseteq A_2$, then the correctable sets extracted from them also satisfy $R(A_1) \subseteq R(A_2)$. Therefore, $\Delta ACC_{max}(A_1) \le \Delta ACC_{max}(A_2)$. This means $\Delta ACC_{max}$ is a monotonically increasing (non-decreasing) function of $|A|$.

By the same logic, since $A \subseteq T$, $\Delta ACC_{max}$ is also a monotonically increasing (non-decreasing) function of $|T|$.
\end{proof}

\begin{theorem}[AR Law: ACC Improvement Variation with Weight]\label{thm:ar_law_weight_variation}
Building on Theorem \ref{thm:ar_law_improvement}, let $\beta = w/(1-w)$.
As the weight $w$ decreases (and thus $\beta$ decreases, while $(1-w)/w$ increases), the ACC improvement magnitude initially increases monotonically.

If for any $r \in R$, Assumption \ref{ass:2} is satisfied (implying $R = A = T$ under ideal conditions), then there exists an optimal weight range where $\beta = \alpha$ such that the ACC improvement reaches its maximum. Subsequently, as $w$ further decreases ( $\beta$ further decreases), the ACC improvement magnitude monotonically decreases, until $\beta = D$, after which ACC begins to decline (That is, the lifting domain of ACC is $[D, \infty]$).
\end{theorem}
\begin{proof}
From Theorem \ref{thm:ar_law_improvement}, we have the condition for favorable exchange:
$$\frac{q_{a,i_a} - q_{a,k_a}}{q'_{a,k_a} - q'_{a,i_a}} < \frac{1-w}{w} < \frac{q_{f,k_f} - q_{f,j_f}}{q'_{f,j_f} - q'_{f,k_f}}$$

Let $A'$ be the set of points satisfying $\frac{q_{a,i_a} - q_{a,k_a}}{q'_{a,k_a} - q'_{a,i_a}} < \frac{1-w}{w}$ and also satisfying the strict exchange condition. It is clear that the current ACC improvement is proportional to $|A'|$, and $A' \subseteq R$.

With $A$ (and thus $R$) fixed, as $w$ decreases, $\beta = w/(1-w)$ decreases, and $\tau = (1-w)/w$ increases. As $\tau$ increases, more samples in $R$ will satisfy $ET < \tau$, so $|A'|$ monotonically increases. Consequently, the ACC improvement magnitude increases.

If $R=A=T$ (ideal case where all potential corrections satisfy Assumption \ref{ass:2}), then when $|A'|=|R|$ (all samples in $R$ are corrected), the ACC improvement reaches its maximum. At this point, let the corresponding $\beta$ be $\alpha$.

As $w$ continues to decrease ( $\beta$ continues to decrease, $\tau$ continues to increase), $|A'|$ no longer increases (it is capped at $|R|$). Now, $R=T=A'$. The increase in $\tau$ will only cause more samples in $F$ to satisfy
$$\frac{q_{f,k_f} - q_{f,j_f}}{q'_{f,j_f} - q'_{f,k_f}} < \frac{1-w}{w}$$
Let $B$ be the set of these positions from $F$. For $f \in B$, an incorrect prediction $j_f$ by SM replaces the correct LM prediction $k_f$, causing the ACC improvement magnitude to decrease (i.e., net ACC starts to fall).

When $|B|$ (number of newly incorrect predictions) becomes greater than $|R|$ (number of correct predictions that were maintained or newly made), the net change in ACC becomes negative, meaning ACC starts to decline rather than improve. We define the weight ratio $\beta$ at which this first occurs as $D$.
\end{proof}

\subsubsection{Corollaries}
\label{A.3}
\begin{corollary}[Multi-dimensional Extension]\label{cor:multi_dim_ext}
The conclusions of the AR Law (Theorems \ref{thm:ar_law_improvement}, \ref{thm:ar_law_max_improvement}, \ref{thm:ar_law_weight_variation}) hold for any dimension of the output vector. In other words, the validity of these laws does not depend on the total number of dimensions in the output space.
\end{corollary}
\begin{proof}
We use mathematical induction on the total number of output dimensions $D_{dim} = |V|$.

Base Case: When $D_{dim}=2$, the problem is a binary classification. The LM's prediction $i$ and SM's prediction $k$ (if different) are the only two classes. All comparisons regarding Exchange Thresholds (ET) and prediction stratification (Assumption \ref{ass:2}) are made directly on these two dimensions. The logic and derivations of Theorems \ref{thm:ar_law_improvement}, \ref{thm:ar_law_max_improvement}, \ref{thm:ar_law_weight_variation} hold directly in this simplest case.

Inductive Hypothesis: Assume that for any output space of dimension $D_{dim}-1$ (where $D_{dim} > 2$), the AR Law holds.

Inductive Step: We need to prove that under this assumption, the AR Law also holds for an output space of dimension $D_{dim}$.

Consider a $D_{dim}$-dimensional output space $V$. The core mechanism of the AR Law, i.e., the change in prediction of the fused model $q_{new}$, depends on the exchange condition in Lemma \ref{lemma:exchange_condition}. For example, for the prediction to flip from class $i$ to class $j$, it must satisfy:
$$w \cdot q_{s,j} + (1-w) \cdot q'_{s,j} > w \cdot q_{s,i} + (1-w) \cdot q'_{s,i}$$

Rearranging gives the condition for the exchange threshold:
$$\frac{q_{s,i} - q_{s,j}}{q'_{s,j} - q'_{s,i}} < \frac{1-w}{w}, \quad \text{where } q'_{s,j} > q'_{s,i}$$
Crucially, this core condition for flipping only involves the probabilities of these two dimensions, $i$ and $j$, and is independent of the probability values of any other dimension $u \in V \setminus \{i,j\}$.

Now, let us arbitrarily remove one dimension $u$ from the $D_{dim}$-dimensional space, forming a $(D_{dim}-1)$-dimensional subspace $V' = V \setminus \{u\}$. We can construct a new classification subproblem defined on $V'$. For this subproblem, we can define a new set of probability distributions normalized over $V'$ to determine which class to output.

Importantly, for any two dimensions still in $V'$ (e.g., $i$ and $j$), their relative probability magnitudes in the original models $q$ and $q'$, and the value of the exchange threshold $ET(s, i \to j)$ calculated from them, are independent of whether we consider dimension $u$.

Therefore, in this $(D_{dim}-1)$-dimensional subspace $V'$, all prerequisites of the AR Law (such as Assumptions \ref{ass:1} and \ref{ass:2}) and core mechanisms (such as Lemma \ref{lemma:exchange_condition}) apply equally. Hence, the AR Law holds for this $(D_{dim}-1)$-dimensional subproblem.

Since the removed dimension $u$ was chosen arbitrarily, this implies that the intrinsic logic of the AR Law holds on any subset of dimensions. Thus, the validity of the law is not restricted to a specific number of dimensions; it naturally extends from $D_{dim}-1$ dimensions to $D_{dim}$ dimensions.

By induction, the AR Law holds for any output space of dimension $D_{dim} \ge 2$.

Specifically, if the removed dimension is the output dimension of the fused model itself at that position, the output changes according to the AR Law. If not, the output at this position does not change.
\end{proof}

\begin{corollary}[Variant based on Assumption \ref{ass:2}]\label{cor:variant_ass2}
If, in set $T$, the number of positions satisfying Assumption \ref{ass:2} and the strict exchange condition (denoted as set $R$) is greater than the number of positions not satisfying Assumption \ref{ass:2} (denoted as $T \setminus A$), i.e., $|R| > |T \setminus A|$, then Theorem \ref{thm:ar_law_improvement} and Theorem \ref{thm:ar_law_max_improvement} still hold.
\end{corollary}
\begin{proof}
The goal of this proof is to find a fusion weight $w$ that achieves a net improvement in accuracy under the condition that Assumption \ref{ass:2} does not fully hold, by utilizing the premise $|R| > |T \setminus A|$.

The core challenge is that "risky" samples $a' \in T \setminus A$ might have an exchange threshold $ET(a')$ greater than the exchange threshold $ET(f)$ of some samples $f \in F$, which means the threshold separation relied upon in the proof of Theorem \ref{thm:ar_law_improvement} is no longer absolutely guaranteed.

However, the premise $|R| > |T \setminus A|$ allows us to adopt a conservative strategy to mitigate this risk. We can construct a safe benign sample subset $R_{safe} \subseteq R$, with its size defined as $|R_{safe}| = |R| - |T \setminus A|$.

According to the premise, this subset contains at least one element. Since all elements of $R_{safe}$ are from $R$, they must satisfy Assumption \ref{ass:2}. Therefore, for any $r \in R_{safe}$ and $f \in F$, their exchange thresholds satisfy $ET(r) < ET(f)$. This property ensures that the exchange thresholds of all samples in $R_{safe}$ are strictly smaller than the exchange thresholds of all samples in $F$.

Thus, we can always choose a weight ratio $\tau = \frac{1-w}{w}$ such that it falls between the threshold ranges of these two sets, i.e., satisfying $\max_{r \in R_{safe}}\{ET(r)\} < \tau < \min_{f \in F}\{ET(f)\}$.

For a weight $w$ satisfying this condition, all samples in $R_{safe}$ will be successfully corrected, yielding an accuracy gain of at least $\frac{|R_{safe}|}{|S|} > 0$, while ensuring that samples in set $F$ are not incorrectly changed, resulting in zero loss. Thus, the accuracy of the fused model achieves a net improvement, proving that Theorem \ref{thm:ar_law_improvement} still holds under this variant condition.

The proof for Theorem \ref{thm:ar_law_max_improvement} is then evident.
\end{proof}

\begin{corollary}[With SM as the Primary Model]\label{cor:sm_primary}
All the preceding theorems are proven with the LM as the primary model, which is generally due to LM having higher ACC than SM.

If the opposite situation occurs, where the small model $q'$ has higher accuracy than the large model $q$, then by swapping the roles of the two models, and under the satisfaction of the assumptions, it can similarly be proven that there exists a fusion weight $w$ such that the accuracy of the fused model is higher than that of the original SM model $q'$.
\end{corollary}
\begin{proof}
The proof of this corollary is based on the principle of symmetry. When we consider the more accurate small model $q'$ as the primary model and $q$ as the auxiliary model, we aim to prove that the accuracy of the fused model can exceed that of $q'$.

To do this, we need to redefine the sets $T$ and $F$ that form the basis of the theory. The new set $T'$, i.e., positions where the primary model ($q'$) predicts incorrectly and the auxiliary model ($q$) predicts correctly, is defined identically to the original set $F$. Similarly, the new set $F'$, i.e., positions where the primary model ($q'$) predicts correctly and the auxiliary model ($q$) predicts incorrectly, is defined identically to the original set $T$.

Therefore, swapping roles in the proof process is equivalent to systematically interchanging $T \leftrightarrow F$ in the original proofs. Considering that the mathematical structure of the linear fusion model $q_{new}$ and the calculation of its exchange threshold are inherently symmetrical with respect to $q$ and $q'$, we can conclude that by performing the aforementioned symmetrical substitutions of roles, sets, and assumptions, the proof chains of all theorems remain unchanged.
\end{proof}

\begin{corollary}[Extension of AR Law to Multi-Model Ensembles]
\label{cor:multi_model_ar_law}
Assume that every pairwise combination of multiple models satisfies Assumption \ref{ass:1}. If the initial weights of each model form a geometric progression, the principles of the two-model ensemble theorem can be naturally extended to multi-model ensembles.
\end{corollary}

\begin{proof}
Assume that there are $m$ models participating in the ensemble and these models are sorted by performance from lowest to highest. The exchange condition then becomes:
$$
\sum_{h=1}^{m} w_h \cdot q_{h,s,j} > \sum_{h=1}^{m} w_h \cdot q_{h,s,i}
$$
Here, $q_{h,s,j}$ denotes the output of the $h$-th model, and we have $ACC_{q_{x_2}} > ACC_{q_{x_1}}$ for any $x_2 > x_1$.

This expression can be decomposed into $C_m^2$ pairwise model combinations:
$$
w_{x_2} \cdot q_{x_2,s,j} + w_{x_1} \cdot q_{x_1,s,j} > w_{x_2} \cdot q_{x_2,s,i} + w_{x_1} \cdot q_{x_1,s,i}
$$

For each pairwise combination, as established previously, the AR Law holds. This means there exists a "safe window" for the weight ratio, defined by a lower and upper bound:
$$
\frac{q_{x_2, a, i_a} - q_{x_2, a, k_a}}{q_{x_1, a, k_a} - q_{x_1, a, i_a}} < \frac{w_{x_1}}{w_{x_2}} < \frac{q_{x_2, f, k_f} - q_{x_2, f, j_f}}{q_{x_1, f, j_f} - q_{x_1, f, k_f}}
$$
Let's denote the weight ratio $\frac{w_{x_1}}{w_{x_2}}$ as $k(x)$, and the lower and upper bounds as $LL(x)$ and $UL(x)$ respectively, where $x$ is the index for the combination, ranging from $1$ to $C_m^2$.

It is evident that if the pairwise exchange condition holds for all combinations, the overall condition must also hold (though the reverse is not necessarily true). Continuing with the logic of the previous proof, we can control the ratio $k(x)$. When $k(x) < \min(UL(x))$, as its value increases, the number of possible correct substitutions for samples in set T will increase, while incorrect substitutions for samples in set F will not occur.

Since the model weights form a geometric progression, we can set $k(x) \in [r, r^{m-1}]$, where $r$ is a constant common ratio. This results in the weights $\{w_h\}$ forming an increasing geometric sequence. Given the constraint $\sum_{h=1}^{m} w_h = 1$, we can solve the base weight:
$$
w_1 = \frac{1-r}{1-r^m}
$$

This initial weight distribution relates $k(x)$ to $r$. Therefore, we can control $k(x)$ for all pairwise combinations by adjusting $r$. The role of $r$ here becomes equivalent to the weight ratio $\frac{1-w}{w}$ in the original two-model proof.

Consequently, we only need to set the upper bound as $\min(UL(x))$. As long as we choose an $r$ such that:
$$
LL(x) < r < \min(UL(x))
$$
we can similarly prove analogous theorems for the multi-model ensemble case.

For the case of a strict increase in accuracy, we also need to choose $r$ such that $\max(LL(x)) < r < \min(UL(x))$ to ensure that the overall exchange condition is met definitively.

\end{proof}

\subsubsection{Concluding Remarks}
\label{A.4}
Theorems \ref{thm:ar_law_improvement}, \ref{thm:ar_law_max_improvement}, and \ref{thm:ar_law_weight_variation} are collectively termed the AR Law. It reveals the principles governing how the accuracy of a fused model changes under certain assumptions. We have observed in tens of thousands of experiments that these laws hold, and combined with some existing research findings, we believe the assumptions of the AR Law are generally valid.

In fact, from the above theorems, it can be seen that the main issues of the AR framework are twofold: finding a method to measure the difference in judgment between two models, and introducing a quantity that reflects our prior judgment of the capabilities of the two models to measure when the difference reaches a threshold for "Accept" and "Reject". This allows us to accept good changes and reject bad ones. In the method described above, this is achieved through probabilities to reflect judgment differences, and the weight $w$ (or $\beta$) reflects the prior judgment.

What we demonstrate in the appendix is the difference in judgments reflected through similarity, and the a priori judgment expressed through the exponent $p$. According to our experimental observations, the methods in the appendix based on the AR framework are also strictly in accordance with the AR law we proposed.

\subsection{Discussion and experiments in Chapters Three and Four.}
\label{B}
\subsubsection{Larger models have more optional tokens}
\label{B.1}
\begin{figure}[H]
    \centering
    \includegraphics[width=\textwidth]{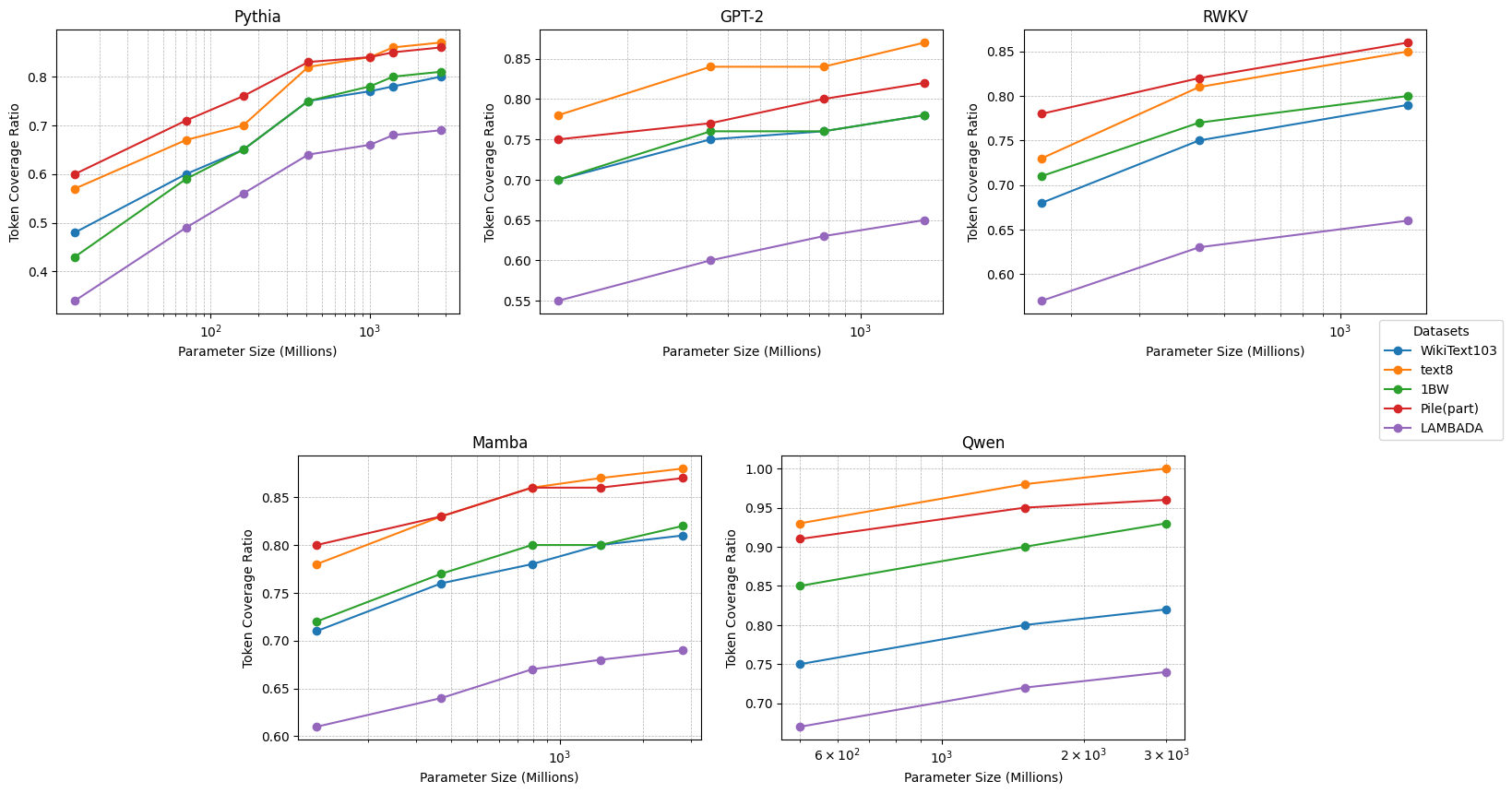}
    \caption{Proportion of tokens generated by scaling law model with different parameters to the total tokens of the dataset under different datasets.}
    \label{fig:both-models}
\end{figure}

The number of tokens that the model may generate across all datasets increases with the increase in the model parameters.

\subsubsection{Variation of Optimal Weight Ratio \texorpdfstring{$\alpha$}{alpha} with Parameter Ratio \texorpdfstring{$k$}{k}}
\label{B.2}
\begin{figure}[H]
    \centering
    \includegraphics[width=\textwidth]{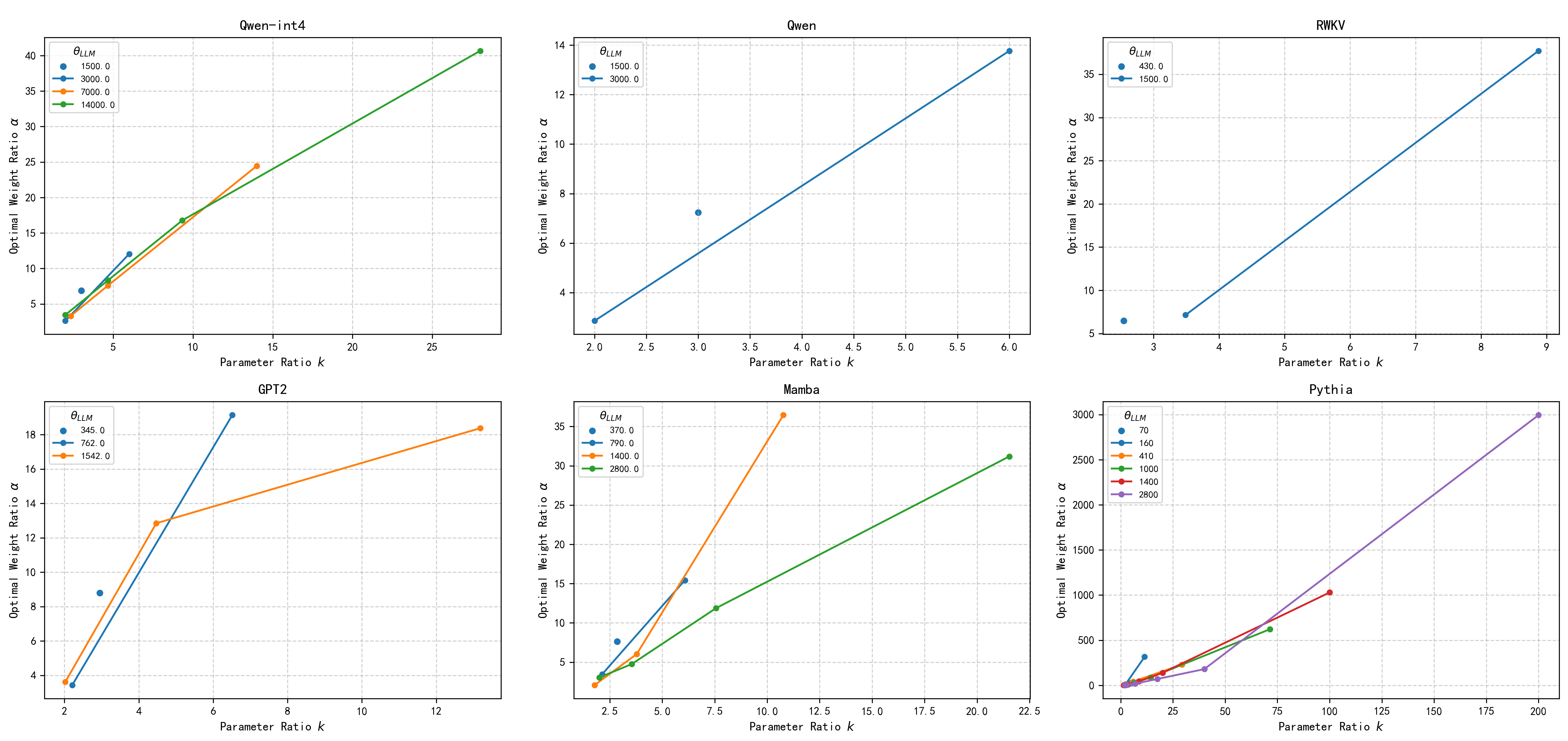}
    \caption{Variation of Optimal Weight Ratio $\alpha$ with Parameter Ratio $k$ for Fixed Main Parameter $\theta_{LLM}$.}
    \label{fig:both-models}
\end{figure}
Figure 4 shows that under the condition of the same main parameters $\theta_{LLM}$, the optimal weight ratio $\alpha$ for each model shows an increasing trend with increasing parameter ratio.

\subsubsection{AR framework based on cosine similarity}
\label{B.3}
As discussed in the main text, the AR framework emphasizes measuring judgment differences between models and ensembling relative overfitting information. Here, we present an alternative approach that, although not entirely successful, provides valuable insights into relative overfitting and the AR law.

We extract the embedding spaces from two models and compute the cosine similarity between each pair of tokens, storing these values for efficient retrieval during inference (see Algorithm1). We prioritize the LLM's output during inference while using the SLM as a reference. Specifically, we obtain the probability distribution from the LLM and identify the token with the highest probability from the SLM. We then retrieve the cosine similarity vector between each token in the LLM's vocabulary and the SLM's highest probability token. Since these similarity values do not form a strict probability space as in our primary method, we employ element-wise multiplication to combine this similarity vector with the LLM's output probabilities. To incorporate the relative overfitting information, we introduce a hyperparameter $p$ as the power exponent for the similarity vector (Algorithm2).

\begin{algorithm}
\caption{Computing Cosine Similarity Matrix}
\begin{algorithmic}[1]
\Require model, chunkSize
\Ensure packedData, metadata
\State \Comment{$E$: Embedding matrix, $V$: Vocabulary size, $S$: Similarity matrix}
\State $E \gets \text{ExtractEmbeddings}(model)$
\State $E_{norm} \gets \frac{E}{\|E\|_2}$ \Comment{Normalize embeddings}
\State $V \gets E_{norm}.\text{shape}[0]$
\State $S \gets \text{zeros}(V, V)$ 
\State Compute similarity matrix by iterating through chunks:
\For{$i \gets 0$ to $V-1$ \textbf{step} chunkSize}
    \For{$j \gets i$ to $V-1$ \textbf{step} chunkSize}
        \State $S_{chunk} \gets E_i \cdot E_j^T$
        \If{$i = j$}
            \State $S[i:i+chunkSize, i:i+chunkSize] \gets \text{upperTriangular}(S_{chunk})$
        \Else
            \State $S[i:i+chunkSize, j:j+chunkSize] \gets S_{chunk}$
        \EndIf
    \EndFor
\EndFor
\State $S_{quant} \gets \lfloor S \cdot 15 \rfloor$ \Comment{Quantize to 4-bit representation}
\State \Return PackUpperTriangular$(S_{quant})$ \Comment{Store in compressed format}
\end{algorithmic}
\end{algorithm}

\begin{algorithm}
\caption{Similarity-Based Model Fusion}
\begin{algorithmic}[1]
\Require $model_1$, $model_2$, input, packedData, metadata, $p$ (power law hyperparameter)
\Ensure Ensemble prediction probabilities
\State \Comment{$L$: Logits, $P$: Probabilities, $T$: Token indices, $S$: Similarity matrix}
\State $S \gets \text{UnpackUpperTriangular}(packedData, metadata)$ \Comment{Decompress similarity matrix}
\State $L_1 \gets model_1(input)$
\State $L_2 \gets model_2(input)$ \Comment{Obtain logits from both models}
\State $flat\_L_1 \gets \text{reshape}(L_1, [-1, V])$
\State $flat\_L_2 \gets \text{reshape}(L_2, [-1, V])$
\State $T_2 \gets \arg\max(flat\_L_2, \text{dim}=1)$ \Comment{Identify model 2's predicted tokens}
\State $simWeights \gets S[:, T_2]^p$ \Comment{Calculate similarity weights with power scaling}
\State $P_1 \gets \text{softmax}(flat\_L_1, \text{dim}=1)$
\State $P_{ensemble} \gets P_1 \odot simWeights$ \Comment{Fuse probabilities using similarity weights}
\State $P_{ensemble} \gets \frac{P_{ensemble}}{\sum P_{ensemble}}$ \Comment{Re-normalize to ensure valid probabilities}
\State \Return $P_{ensemble}$
\end{algorithmic}
\end{algorithm}
This approach is considered less successful because, despite implementing various optimization techniques, the storage requirements for the similarity matrix remain substantial without yielding performance improvements—in fact, it slightly underperforms compared to the method described in the main text. Furthermore, this approach still necessitates searching for appropriate hyperparameters under the AR law to ensemble the relative overfitting information between models effectively.

The primary value of discussing this method lies in exploring the AR law within the AR framework. Our experiments confirm that this approach still adheres to the pattern we proposed. Specifically, there exists an optimal value $p_0$ such that when $p < p_0$, performance gains increase with increasing $p$, while when $p > p_0$, performance gains decrease with increasing $p$. This observation further validates our theoretical framework.

The limited effectiveness of this more complex approach may be attributed to the fact that using similarity metrics and power-law adjustments as measures of judgment differences between models introduces additional complexity. Using only the highest probability token from the SLM may be insufficient to fully capture the nuanced relationships between models. This approach could achieve better performance improvements with more research and more sophisticated designs.

During our experimental process, we discovered that without introducing performance-related overfitting difference information between two models based on the AR framework, specifically the power-law hyperparameter p, we cannot guarantee that performance will consistently improve rather than decline. This further confirms our conclusion that the AR law, based on the AR framework and relative overfitting concept, is a prerequisite to ensure universal improvement rather than deterioration.

\clearpage
\newpage
\subsection{Experiments in the rest of the machine learning field}
\label{C}
In this appendix, we detail the experimental setups employed. For image recognition, a suite of pre-trained ResNet models of varying sizes (ResNet-18 to ResNet-152) was utilized. Subsequently, these models were fine-tuned for one epoch in the CIFAR-10, CIFAR-100\cite{67}, and Tiny ImageNet-200 datasets \cite{59, 60, 61, 62, 63} to assess classification efficacy across varying complexity levels. Currently, the ESM-2 model was used for the protein sequence prediction task. This task focuses on predicting protein sequences, which holds significant societal importance, including applications such as the design of novel protein sequences for therapeutic or industrial purposes. The predictive capabilities of the model were benchmarked in the UniRef50\cite{66} data set.

The experimental results demonstrate that our proposed ideas and framework can be universally applied to other areas of machine learning.

\subsubsection{Computer Vision}
\label{C.1}
Our experiments employed a consistent post-training approach across CIFAR-10, CIFAR-100, and Tiny-ImageNet datasets. We fine-tuned ResNet series models (from ResNet18 to ResNet152) pre-trained on ImageNet. The fine-tuning process began by replacing the final fully connected layer, adjusting the output dimensions to match the respective datasets: 10 classes for CIFAR-10, 100 for CIFAR-100, and 200 for Tiny-ImageNet. All images were resized to 224×224 pixels and normalized using ImageNet parameters (mean=[0.485, 0.456, 0.406], std=[0.229, 0.224, 0.225]). We utilized SGD optimizer with a learning rate 0.001 and momentum of 0.9, typically running for one epoch. Performance was evaluated using Top-1 accuracy metrics for both individual models and ensembles.

\begin{table}[htbp]
\caption{Top-1 Accuracy of ResNet models and AR framework ensembles on CIFAR-10, CIFAR-100, and Tiny ImageNet. Accuracies (\%) and relative improvements (\%) are shown. Improvements are calculated for the Large Models (LM) when combined with the Small Models (SM), relative to the best-performing single model's Top-1 performance.}
\centering
\begin{tabular}{cl ccc}
  \toprule
  \multicolumn{2}{c}{Model Configuration} & CIFAR-10 & CIFAR-100 & Tiny ImageNet \\
  \midrule
  \multicolumn{5}{c}{Single Model Baseline Performance (Top-1 Acc. (\%))} \\
  \midrule
  \centering resnet18 & -- & 90.44 & 55.47 & 38.69 \\
  \centering resnet34 & -- & 93.14 & 64.99 & 49.66 \\
  \centering resnet50 & -- & 93.69 & 65.54 & 55.19 \\
  \centering resnet101 & -- & 95.31 & 71.46 & 64.61 \\
  \centering resnet152 & -- & 95.39 & 72.80 & 67.18 \\
  \midrule
  \multicolumn{5}{c}{AR Framework Ensemble: Top-1 Improvement relative to LM Baseline (\%)} \\
  \midrule
  LM & SM & \multicolumn{3}{c}{Top-1 Imp. (\%) ↑} \\
  \midrule
  \centering resnet34 & resnet18 & 0.61\% & 1.55\% & 2.07\% \\
  \midrule
  \centering resnet50 & resnet18 & 0.21\% & 1.19\% & 0.49\% \\
  \centering resnet50 & resnet34 & 0.81\% & 5.00\% & 3.26\% \\
  \midrule
  \centering resnet101 & resnet18 & 0.17\% & 0.57\% & 0.01\% \\
  \centering resnet101 & resnet34 & 0.41\% & 1.43\% & 0.12\% \\
  \centering resnet101 & resnet50 & 0.39\% & 1.33\% & 0.51\% \\
  \midrule
  \centering resnet152 & resnet18 & 0.52\% & 0.84\% & 0.01\% \\
  \centering resnet152 & resnet34 & 0.29\% & 1.70\% & 0.06\% \\
  \centering resnet152 & resnet50 & 0.43\% & 1.48\% & 0.13\% \\
  \centering resnet152 & resnet101 & 0.85\% & 2.90\% & 2.22\% \\
  \bottomrule
\end{tabular}
\end{table}

\clearpage
\newpage
\subsubsection{AI For Science}
\label{C.2}
We evaluated different scales of ESM-2 protein language models (8M to 3B parameters) using the Masked Language Modeling (MLM) task. The experiment employed a 15\% random masking rate on amino acid residues while avoiding special tokens ([CLS], [SEP], [PAD], etc.). Performance was measured using perplexity (PPL), calculated as the exponentiated cross-entropy loss on masked positions. For excessively large loss values (>70), we returned infinity to prevent numerical overflow. This evaluation methodology effectively reflects the protein sequence modeling capabilities of ESM-2 models across different parameter scales.
\begin{table}[htbp]
\caption{Perplexity (PPL) of ESM-2 models and AR framework ensembles on UniRef50 (test\_00.csv). Baseline PPLs and relative PPL reductions (\%) for ensembles are shown. PPL reductions are calculated for the LM when combined with the SM, relative to the best performing single baseline PPL.}
\centering
\begin{tabular}{llc}
  \toprule
  \multicolumn{2}{c}{\textbf{Model Configuration}} & \textbf{PPL} $\downarrow$ \\
  \midrule
  \multicolumn{3}{c}{\textbf{Single Model Baseline Performance (PPL)}} \\
  \midrule
  esm2\_8m   & -- & 12.30 \\
  esm2\_35m  & -- & 10.65 \\
  esm2\_150m & -- & 9.17  \\
  esm2\_650m & -- & 7.59  \\
  esm2\_3b    & -- & 6.65  \\
  \midrule
  \multicolumn{3}{c}{\textbf{AR Framework Ensemble: PPL Reduction relative to LM Baseline (\%)}} \\
  \midrule
  LM & SM & PPL Reduction (\%) $\uparrow$ \\
  \midrule
  esm2\_35m  & esm2\_8m   & 0.56\% \\
  \midrule
  esm2\_150m & esm2\_8m   & 0.12\% \\
  esm2\_150m & esm2\_35m  & 0.36\% \\
  \midrule
  esm2\_650m & esm2\_8m   & 0.17\% \\
  esm2\_650m & esm2\_35m  & 0.14\% \\
  esm2\_650m & esm2\_150m & 0.01\% \\
  \midrule
  esm2\_3b    & esm2\_8m   & 0.17\% \\
  esm2\_3b    & esm2\_35m  & 0.17\% \\
  esm2\_3b    & esm2\_150m & 0.18\% \\
  esm2\_3b    & esm2\_650m & 0.78\% \\
  \bottomrule
\end{tabular}
\end{table}

The experimental results of using ResNet for image classification and ESM2 for protein sequence prediction demonstrate that the straightforward method proposed based on the AR framework is still generally effective in fields beyond NLP. However, it is important to note that in nonlanguage modeling problems, the performance of the AR law does not exhibit the same strictness as it does in language modeling problems. Although the overall trend aligns with the AR law, small fluctuations may occur locally. This is because the AR framework is based on relative overfitting for correction, which is intuitive in fundamental language modeling problems, but exhibits more complexity in other scenarios. Therefore, this paper posits that further in-depth research on relative overfitting and the AR law is necessary.

\clearpage
\newpage

\subsection{Complementary Experiments}
\label{D}

\subsubsection{Supplementary Experiments of Chapter 4.1}
\label{D.1}
In Appendix D, we demonstrate the performance of Pythia, RWKV, Mamba, and Qwen (based on int4 quantization) within the AR framework, which is based on the AR framework, on the fundamental language modeling benchmarks WikiText2, WikiText103, text8, and Pile (with Pile being a randomly sampled portion of the test set), as well as the 1BW benchmark. The improvements under our framework are generally practical and largely conform to the AR law we proposed.

Qwen is a series of large language models developed by Alibaba Cloud in 2023. As a prominent commercial large language model, it competes with other industry models like GPT and Claude. The model architecture is based on the Transformer decoder framework, with optimizations for context length and instruction-following capabilities. We used five different sizes of the Qwen model: 0.5b, 1.5b, 3b, 7b, and 14b. We will present the quantization results based on int4 in Table 8 to facilitate formatting.

\begin{table}[H]
\caption{The table lists the test results of the commercial LLM Qwen based on our proposed AR framework in fundamental language benchmarks.}
\centering
\resizebox{\textwidth}{!}{
\begin{tabular}{llcccccc}
  \toprule
  LLM & SLM & WikiText2 & WikiText103 & text8 & 1BW & Pile(part) & LAMBADA \\
  Qwen & Qwen & PPL $\downarrow$ & PPL $\downarrow$ & PPL $\downarrow$ & BPC $\downarrow$ & PLL $\downarrow$ & PPL $\downarrow$ \\
  \midrule
  \multicolumn{8}{c}{\textbf{Single Model Baseline Performance}} \\
  \midrule
  \multicolumn{2}{c}{0.5b} & 12.41 & 14.56 & 1.00 & 40.87 & 10.74 & 37.22 \\
  \multicolumn{2}{c}{1.5b} & 9.48 & 10.49 & 0.86 & 30.66 & 8.15 & 29.17 \\
  \multicolumn{2}{c}{3b} & 8.51 & 9.19 & 0.80 & 26.79 & 7.30 & 26.23 \\
  \midrule
  \multicolumn{8}{c}{\textbf{Improvement of the evaluation metric relative to the best performing single model (\%)}} \\
  \midrule
  1.5b & 0.5b & 1.47\% & 1.42\% & 0.37\% & 1.21\% & 0.46\% & 1.21\% \\
  \midrule
  3b & 0.5b & 1.10\% & 1.15\% & 0.20\% & 0.66\% & 0.21\% & 0.65\% \\
  3b & 1.5b & 3.18\% & 3.29\% & 1.17\% & 2.50\% & 1.35\% & 2.27\% \\
  \bottomrule
\end{tabular}
}
\end{table}

Receptance Weighted Key Value (RWKV) is a novel architecture that combines the parallelizable training of Transformers with the constant-memory, linear-time inference of RNNs by leveraging a linear attention mechanism. It achieves performance on par with similarly sized Transformers. We used 169M, 430M, and 1.5B of RWKV in our experiments.

\begin{table}[H]
\caption{The table lists the test results of RWKV, a model that fuses RNN and Transformer structures, based on our proposed AR framework in fundamental language benchmarks.}
\centering
\resizebox{\textwidth}{!}{
\begin{tabular}{llcccccc}
  \toprule
  LLM & SLM & WikiText2 & WikiText103 & text8 & 1BW & Pile(part) & LAMBADA \\
  RWKV & RWKV & PPL $\downarrow$ & PPL $\downarrow$ & PPL $\downarrow$ & BPC $\downarrow$ & PLL $\downarrow$ & PPL $\downarrow$ \\
  \midrule
  \multicolumn{8}{c}{\textbf{Single Model Baseline Performance}} \\
  \midrule
  \multicolumn{2}{c}{169m} & 24.10 & 29.19 & 1.23 & 56.70 & 13.31 & 34.91 \\ 
  \multicolumn{2}{c}{430m} & 17.23 & 19.43 & 1.09 & 43.86 & 10.10 & 28.59 \\ 
  \multicolumn{2}{c}{1.5b} & 12.76 & 13.71 & 0.97 & 33.39 & 8.05 & 23.93 \\ 
  \midrule
  \multicolumn{8}{c}{\textbf{Improvement of the evaluation metric relative to the best performing single model (\%)}} \\
  \midrule
  430m & 169m & 1.88\% & 1.11\% & 0.40\% & 1.70\% & 0.65\% & 1.31\% \\
  \midrule
  1.5b & 169m & 0.26\% & 0.17\% & 0.10\% & 0.34\% & 0.28\% & 0.48\% \\
  1.5b & 430m & 1.18\% & 0.74\% & 0.38\% & 1.00\% & 0.62\% & 1.04\% \\
  \bottomrule
\end{tabular}
}
\end{table}

\clearpage
\newpage

Mamba is a linear-time sequence modeling architecture that replaces traditional attention modules with input-conditioned selective state-space model (SSM) layers, enabling content-based reasoning over long sequences in O(n) time while maintaining Transformer-comparable accuracy. Mamba has demonstrated representative performance across modalities such as language, audio, and genomics. In our experiments, we used the 130M, 370M, 790M, 1.4B, and 2.8B variants of Mamba.

\begin{table}[H]
\caption{The table lists the test results of Mamba, a model using the SSM architecture, based on our proposed AR framework in fundamental language benchmarks.}
\centering
\resizebox{\textwidth}{!}{
\begin{tabular}{llcccccc}
  \toprule
  LLM & SLM & WikiText2 & WikiText103 & text8 & 1BW & Pile(part) & LAMBADA \\
  Mamba & Mamba & PPL $\downarrow$ & PPL $\downarrow$ & PPL $\downarrow$ & BPC $\downarrow$ & PLL $\downarrow$ & PPL $\downarrow$ \\
  \midrule
  \multicolumn{8}{c}{\textbf{Single Model Baseline Performance}} \\
  \midrule
  \multicolumn{2}{c}{130m} & 19.59 & 23.26 & 1.14 & 51.04 & 11.70 & 31.71 \\ 
  \multicolumn{2}{c}{370m} & 14.40 & 16.05 & 1.02 & 38.55 & 9.05 & 26.13 \\ 
  \multicolumn{2}{c}{790m} & 12.49 & 13.51 & 0.95 & 32.99 & 7.95 & 23.51 \\ 
  \multicolumn{2}{c}{1.4b} & 11.44 & 12.16 & 0.92 & 30.58 & 7.40 & 22.32 \\ 
  \multicolumn{2}{c}{2.8b} & 10.21 & 10.60 & 0.87 & 27.95 & 6.69 & 20.69 \\ 
  \midrule
  \multicolumn{8}{c}{\textbf{Improvement of the evaluation metric relative to the best performing single model (\%)}} \\
  \midrule
  370m & 130m & 1.33\% & 0.83\% & 0.38\% & 1.00\% & 0.52\% & 1.26\% \\
  \midrule
  790m & 130m & 0.68\% & 0.35\% & 0.20\% & 0.39\% & 0.20\% & 0.76\% \\
  790m & 370m & 2.98\% & 2.29\% & 0.95\% & 1.79\% & 1.13\% & 1.90\% \\
  \midrule
  1.4b & 130m & 0.32\% & 0.18\% & 0.12\% & 0.18\% & 0.10\% & 0.57\% \\
  1.4b & 370m & 1.45\% & 1.07\% & 0.47\% & 0.90\% & 0.45\% & 1.19\% \\
  1.4b & 790m & 3.42\% & 3.00\% & 1.54\% & 3.04\% & 1.82\% & 2.87\% \\
  \midrule
  2.8b & 130m & 0.34\% & 0.26\% & 0.16\% & 0.25\% & 0.15\% & 0.53\% \\
  2.8b & 370m & 1.10\% & 0.80\% & 0.39\% & 0.66\% & 0.31\% & 0.87\% \\
  2.8b & 790m & 1.69\% & 1.31\% & 0.81\% & 1.71\% & 0.71\% & 1.56\% \\
  2.8b & 1.4b & 2.93\% & 2.58\% & 1.32\% & 2.67\% & 1.40\% & 2.30\% \\
  \bottomrule
\end{tabular}
}
\end{table}

\begin{table}[htbp]
\caption{The table lists the test results of Qwen (INT4 quantized), based on our proposed AR framework in fundamental language benchmarks.}
\centering
\resizebox{\textwidth}{!}{%
\begin{tabular}{llcccccc}
  \toprule
  LLM & SLM & WikiText2 & WikiText103 & text8 & 1BW & Pile(part) & LAMBADA \\
  Qwen-int4 & Qwen-int4 & PPL $\downarrow$ & PPL $\downarrow$ & BPC $\downarrow$ & PPL $\downarrow$ & PLL $\downarrow$ & PPL $\downarrow$ \\
  \midrule
  \multicolumn{8}{c}{\textbf{Single Model Baseline Performance}} \\
  \midrule
  \multicolumn{2}{c}{0.5b} & 13.62 & 16.22 & 1.04 & 46.70 & 12.04 & 41.77 \\
  \multicolumn{2}{c}{1.5b} & 10.12 & 11.29 & 0.88 & 32.54 & 8.68 & 30.82 \\
  \multicolumn{2}{c}{3b}   & 9.00  & 9.75  & 0.82 & 28.33 & 7.70 & 27.82 \\
  \multicolumn{2}{c}{7b}   & 7.66  & 8.19  & 0.75 & 24.72 & 6.75 & 24.34 \\
  \multicolumn{2}{c}{14b}  & 6.53  & 6.96  & 0.67 & 21.79 & 6.03 & 21.28 \\
  \midrule
  \multicolumn{8}{c}{\textbf{Improvement of the evaluation metric relative to the best performing single model (\%)}} \\
  \midrule
  1.5b & 0.5b & 2.10\% & 2.10\% & 0.48\% & 1.54\% & 0.78\% & 1.53\% \\
  \midrule
  3b & 0.5b & 1.59\% & 1.50\% & 0.28\% & 1.18\% & 0.44\% & 1.16\% \\
  3b & 1.5b & 4.29\% & 4.39\% & 1.55\% & 4.16\% & 2.37\% & 3.90\% \\
  \midrule
  7b & 0.5b & 0.60\% & 0.60\% & 0.10\% & 0.43\% & 0.09\% & 0.53\% \\
  7b & 1.5b & 1.27\% & 1.38\% & 0.45\% & 1.54\% & 0.52\% & 1.32\% \\
  7b & 3b   & 2.57\% & 2.84\% & 1.07\% & 3.01\% & 1.37\% & 2.28\% \\
  \midrule
  14b & 0.5b & 0.33\% & 0.46\% & 0.11\% & 0.23\% & 0.05\% & 0.24\% \\
  14b & 1.5b & 0.60\% & 0.88\% & 0.33\% & 0.69\% & 0.24\% & 0.49\% \\
  14b & 3b   & 1.06\% & 1.38\% & 0.56\% & 1.17\% & 0.48\% & 0.73\% \\
  14b & 7b   & 2.31\% & 2.99\% & 1.44\% & 2.48\% & 1.37\% & 1.59\% \\
  \bottomrule
\end{tabular}%
}
\end{table}

Pythia is an open-source suite of autoregressive Transformer models developed by EleutherAI to enable rigorous scientific research into large language model behavior and training dynamics, with fully public data provenance and checkpointing. It is a highly classic series of scaling law models. We used the 14M, 70M, 160M, 410M, 1B, 1.4B, and 2.8B variants in our experiments.
\begin{table}[H]
\caption{The table lists the test results of the classic multi-scale scaling law model Pythia based on our proposed AR framework in fundamental language benchmarks.}
\centering
\resizebox{\textwidth}{!}{
\begin{tabular}{llcccccc}
  \toprule
  LLM & SLM & WikiText2 & WikiText103 & text8 & 1BW & Pile(part) & LAMBADA \\
  Pythia & Pythia & PPL $\downarrow$ & PPL $\downarrow$ & PPL $\downarrow$ & BPC $\downarrow$ & PLL $\downarrow$ & PPL $\downarrow$ \\
  \midrule
  \multicolumn{8}{c}{\textbf{Single Model Baseline Performance}} \\
  \midrule
  \multicolumn{2}{c}{14m} & 95.80 & 134.06 & 1.66 & 234.11 & 39.09 & 81.20 \\ 
  \multicolumn{2}{c}{70m} & 39.16 & 52.15 & 1.42 & 100.97 & 19.97 & 49.40 \\ 
  \multicolumn{2}{c}{160m} & 24.20 & 29.73 & 1.26 & 66.36 & 13.89 & 37.17 \\ 
  \multicolumn{2}{c}{410m} & 16.14 & 18.23 & 1.08 & 47.20 & 9.84 & 28.55 \\ 
  \multicolumn{2}{c}{1000m} & 13.71 & 15.03 & 1.02 & 40.08 & 8.59 & 25.72 \\ 
  \multicolumn{2}{c}{1400m} & 12.48 & 13.39 & 0.97 & 35.32 & 7.97 & 24.06 \\ 
  \multicolumn{2}{c}{2800m} & 10.92 & 11.54 & 0.92 & 31.27 & 7.23 & 22.12 \\ 
  \midrule
  \multicolumn{8}{c}{\textbf{Improvement of the evaluation metric relative to the best performing single model (\%)}} \\
  \midrule
  70m & 14m & 0.47\% & 0.38\% & 0.19\% & 0.24\% & 0.27\% & 1.18\% \\
  \midrule 
  160m & 14m & 0.03\% & 0.01\% & 0.01\% & 0.02\% & 0.27\% & 0.39\% \\
  160m & 70m & 0.95\% & 0.51\% & 0.30\% & 0.81\% & 0.78\% & 1.33\% \\
  \midrule 
  410m & 14m & 0.04\% & 0.04\% & 0.01\% & 0.11\% & 0.12\% & 0.41\% \\
  410m & 70m & 0.44\% & 0.14\% & 0.06\% & 0.56\% & 0.18\% & 0.66\% \\
  410m & 160m & 1.15\% & 0.50\% & 0.18\% & 1.38\% & 0.43\% & 1.19\% \\
  \midrule 
  1b & 14m & 0.02\% & 0.02\% & 0.00\% & 0.01\% & 0.09\% & 0.35\% \\
  1b & 70m & 0.18\% & 0.05\% & 0.02\% & 0.17\% & 0.09\% & 0.46\% \\
  1b & 160m & 0.43\% & 0.15\% & 0.07\% & 0.44\% & 0.16\% & 0.69\% \\
  1b & 410m & 2.51\% & 1.91\% & 0.97\% & 1.91\% & 1.02\% & 1.93\% \\
  \midrule 
  1.4b & 14m & 0.01\% & 0.02\% & 0.00\% & 0.00\% & 0.14\% & 0.28\% \\
  1.4b & 70m & 0.10\% & 0.04\% & 0.01\% & 0.05\% & 0.14\% & 0.38\% \\
  1.4b & 160m & 0.20\% & 0.09\% & 0.04\% & 0.16\% & 0.19\% & 0.54\% \\
  1.4b & 410m & 0.98\% & 0.66\% & 0.33\% & 0.65\% & 0.41\% & 0.99\% \\
  1.4b & 1b & 3.17\% & 2.66\% & 1.05\% & 2.24\% & 1.77\% & 2.40\% \\
  \midrule 
  2.8b & 14m & 0.01\% & 0.02\% & 0.00\% & 0.00\% & 0.16\% & 0.20\% \\
  2.8b & 70m & 0.08\% & 2.42\% & 0.01\% & 0.02\% & 0.15\% & 0.26\% \\ 
  2.8b & 160m & 0.15\% & 0.06\% & 0.03\% & 0.07\% & 0.19\% & 0.39\% \\
  2.8b & 410m & 0.36\% & 0.27\% & 0.16\% & 0.30\% & 0.20\% & 0.53\% \\
  2.8b & 1b & 1.07\% & 0.85\% & 0.44\% & 0.82\% & 0.59\% & 1.08\% \\
  2.8b & 1.4b & 1.83\% & 1.68\% & 0.98\% & 1.95\% & 1.10\% & 1.75\% \\
  \bottomrule
\end{tabular}
}
\end{table}
A series of experimental results based on models with different structures and purposes indicates that our proposed AR framework, which is based on relative overfitting, has achieved widespread effectiveness, and the results are in almost complete accordance with the AR law we proposed after tens of thousands of experiments.

\subsubsection{Supplementary Experiments of Chapter 4.2}
\label{D.2}
The utilization of models trained on diverse datasets within the main body of this work is primarily intended to offer a more granular illustration of the relative overfitting phenomenon under various conditions. This section introduces an approach for Large Language Model (LLM) rectification employing a singular auxiliary model. Specifically, we trained and fine-tuned a model across all datasets utilized in our experiments, designed to exhibit performance slightly inferior to the primary model on most benchmarks. This auxiliary model shares the same architecture (a simple Transformer) and parameter count (approximately 180M) as the primary model. Notably, we intentionally engineered its performance on the 1BW dataset to be marginally superior to that of the main model. The corresponding results are presented in Table 10. It is pertinent to mention that the performance achieved on the 1BW dataset is still considerably below its potential optimum. Naturally, such non-generically trained models, through iterative fine-tuning, may undergo complex and multifaceted changes. Consequently, relative overfitting in these instances can be more intricate than models trained systematically according to established scaling laws. A more profound investigation into these complex scenarios remains a subject for future research.

\begin{table}
\caption{All models used the simplest Transformer architecture. For WikiText2, WikiText103, and LAMBADA results, models were trained on the WikiText2 training set, while for 1BW and text8 results, models were trained on the 1BW and text8 training sets, respectively.}

\centering
\begin{tabular}{lccccc}
  \toprule
  \multicolumn{1}{c}{Model} & WikiText2 & WikiText103 & LAMBADA & text8 & 1BW \\
  \multicolumn{1}{c}{Trained Model} & PPL $\downarrow$ & PPL $\downarrow$ & PPL $\downarrow$ & BPC $\downarrow$ & PPL $\downarrow$ \\
  \midrule
  GPT2-Train & 15.24 & 17.74 & 37.63 & 1.12 & 44.21 \\
  Other-Train & 12.85 & 23.97 & 32.14 & 1.04 & 22.79 \\
  \midrule
  \multicolumn{6}{c}{\textbf{Improvement of the evaluation metric relative to the best performing single model (\%)}} \\
  \midrule
  GPT2-1.5b & 55.26\% & 54.44\% & 73.96\% & 25.43\% & 63.84\% \\
  Pythia-2.8b & 58.92\% & 53.51\% & 64.84\% & 33.15\% & 71.89\% \\
  Mamba-2.8b & 57.13\% & 51.52\% & 63.29\% & 31.95\% & 69.91\% \\
  RWKV-1.5b & 62.25\% & 57.45\% & 66.26\% & 34.18\% & 72.98\% \\
  \bottomrule
\end{tabular}
\end{table}

\subsubsection{Supplementary Experiments of Chapter 4.5}
Here, we will present more results regarding the AR law. Experimental results indicate that our integration theory is universally valid. In fact, in all of our experiments, with the exception of a very few occurrences of one in ten thousand, the AR law holds universally.

\begin{figure}[h]
    \centering
    \includegraphics[width=0.7\textwidth]{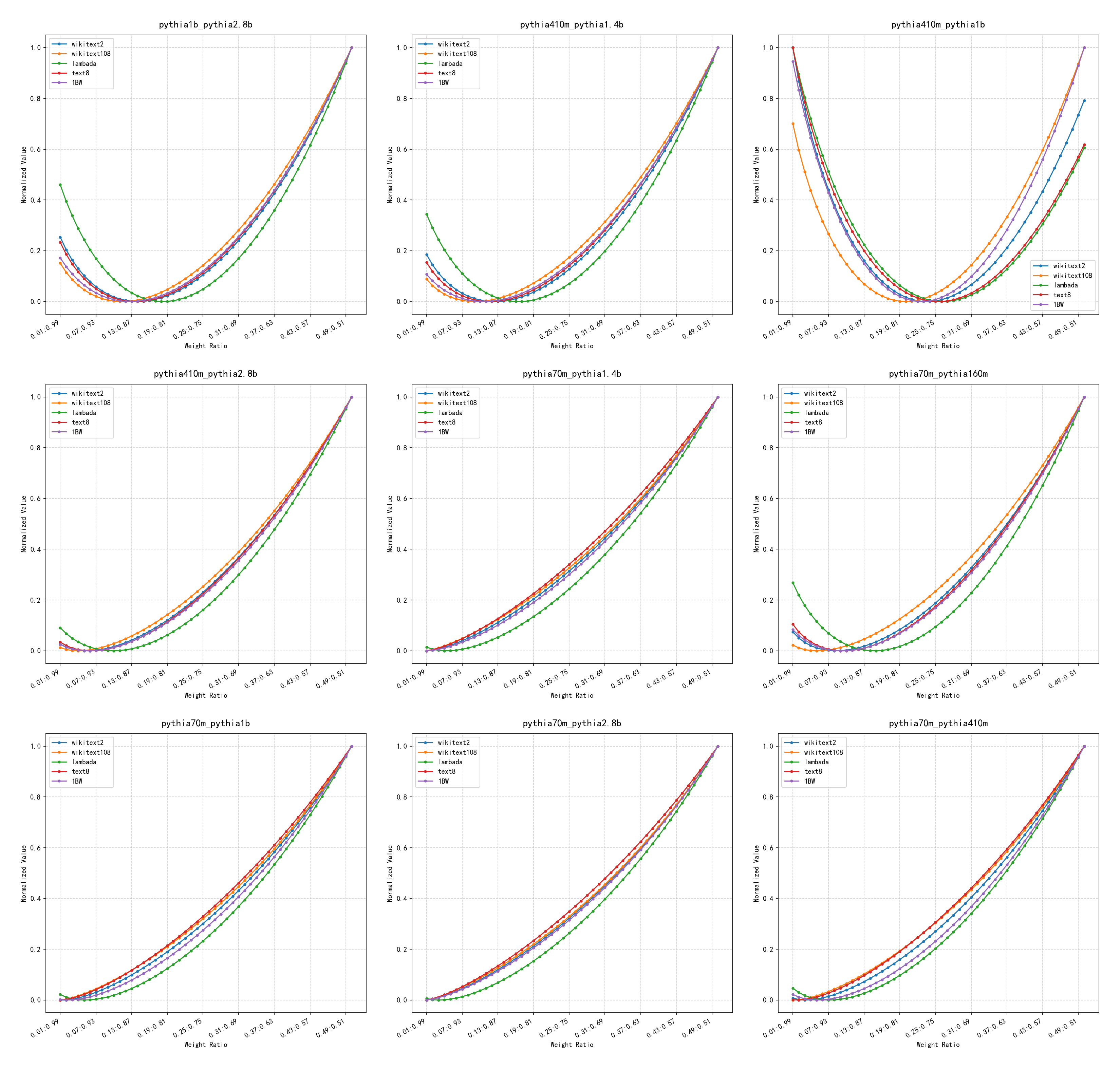}
    \vspace{-0.5cm} 
    \caption{The AR law of the Pythia model.}
    \label{p6}
\end{figure}

\begin{figure}[h]
    \centering
    \includegraphics[width=0.7\textwidth]{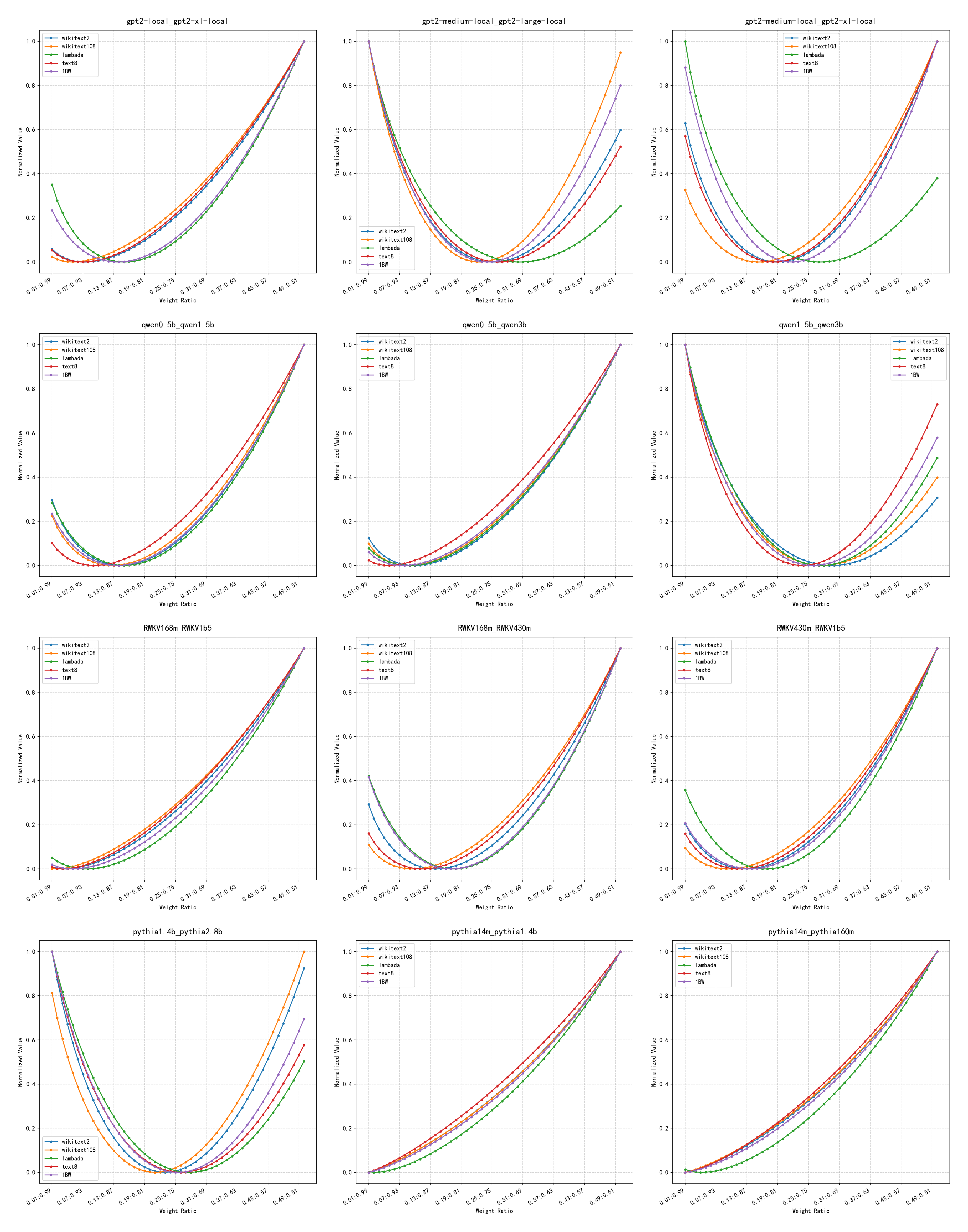}
    \vspace{-0.5cm} 
    \caption{Comparison results of the AR law for four different architectural combination models.}
    \label{p4}
\end{figure}

\begin{figure}[t]
    \centering
    \includegraphics[width=0.7\textwidth]{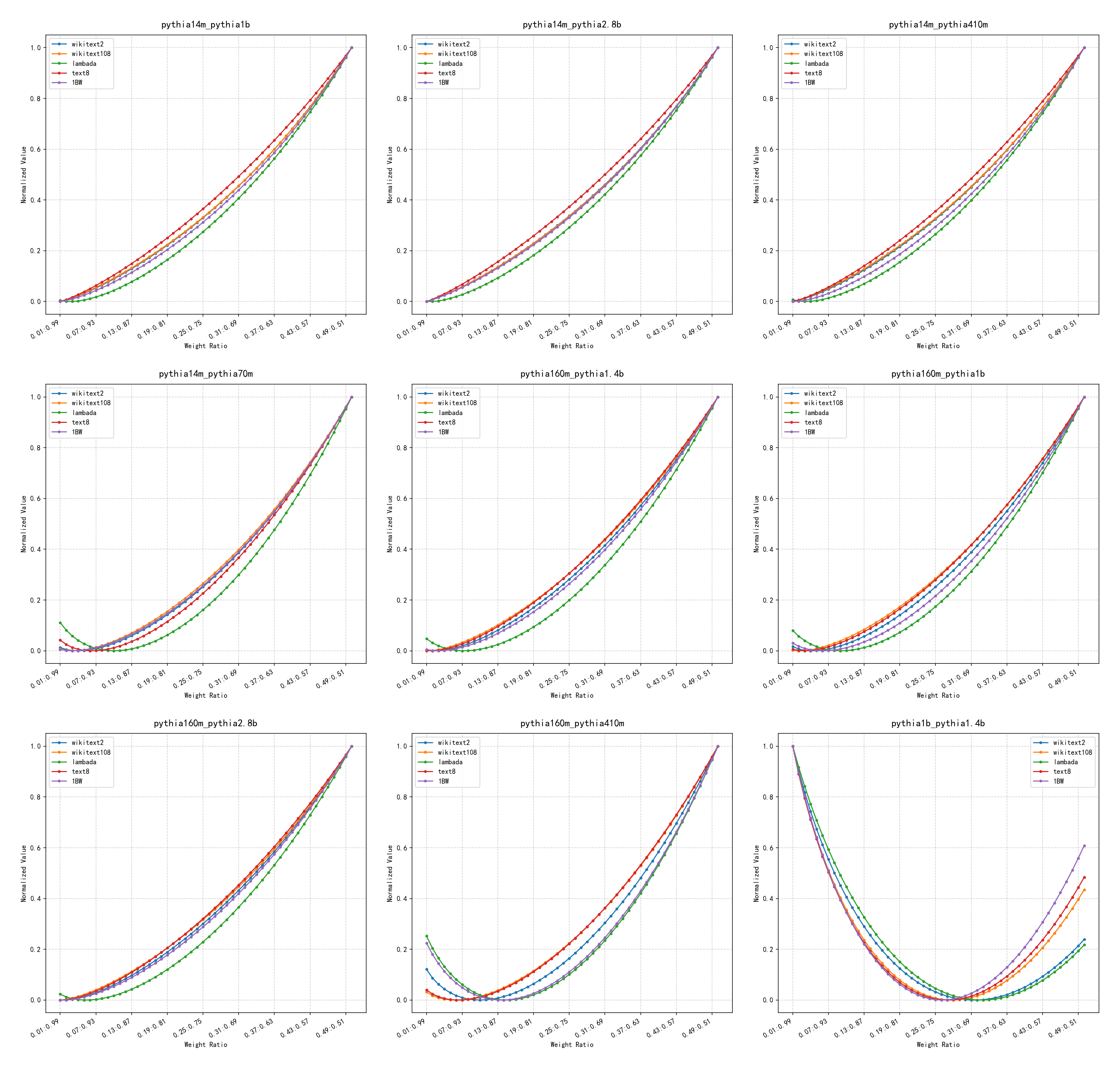}
    \vspace{-0.5cm} 
    \caption{The AR law of the Pythia model.}
    \label{p5}
\end{figure}

\clearpage
\newpage

\subsection{Experimental Setting/Details}
All experiments were conducted on an RTX 4090 graphics card.
\subsubsection{Experimental Section 5.2}

In Section 5.2, the custom models (\texttt{SimpleTransformer}) were trained using tokenizers derived from various reference pre-trained models (e.g., GPT2, Mamba).

The custom model employed a standard transformer encoder architecture with the following parameters:
\begin{itemize}
    \item Transformer Encoder Layers: 6
    \item Attention Heads per Layer: 6
    \item Hidden Size / Embedding Dimension: 384
    \item Feedforward Network Dimension: 1536 ($4 \times \text{hidden\_size}$)
    \item Activation Function: GELU
    \item Dropout Rate: 0.1
    \item Maximum Sequence Length / Positional Embeddings: 1024
    \item Vocabulary Size: Determined by the specific reference tokenizer used for each training run.
\end{itemize}

The training process for these models utilized the following hyperparameters:
\begin{itemize}
    \item Random Seed: 42
    \item Batch Size (per device): 2
    \item Gradient Accumulation Steps: 4 (Effective Batch Size: 8)
    \item Optimizer: AdamW
    \item Learning Rate: 5e-5
    \item LR Scheduler: Cosine schedule with 10\% warmup steps
    \item Gradient Clipping: Maximum norm of 1.0
    \item Mixed Precision Training: Enabled
\end{itemize}

\subsubsection{Experimental Section 5.3}
In Section 5.3, we describe the training of three types of sequential language models: LSTM, GRU, and Transformer.

For each architecture type, separate models were trained by varying the embedding dimension (\texttt{embed\_size}) across the values: **16, 32, 128, 256, 512, 768, and 1024**.

The remaining architectural parameters were fixed based on the script's defaults for each specific architecture type:
\begin{itemize}
    \item \textbf{For LSTM and GRU models:}
        \begin{itemize}
            \item Hidden Size (\texttt{hidden\_size}): 128
            \item Number of Recurrent Layers (\texttt{num\_layers}): 1
        \end{itemize}
    \item \textbf{For Transformer models:}
        \begin{itemize}
            \item Number of Attention Heads (\texttt{nhead}): 2
            \item Number of Transformer Encoder Layers (\texttt{num\_layers}): 1
            \item Sequence Length for Positional Encoding (\texttt{seq\_len}): 35
        \end{itemize}
    \item \textbf{Common to all architectures:}
        \begin{itemize}
            \item Dropout Rate: 0.2 (applied in relevant layers)
        \end{itemize}
\end{itemize}
Note: The vocabulary size (\texttt{vocab\_size}) was constructed based on the training data (WikiText2 train split), capped at a maximum of 10,000 tokens and requiring a minimum frequency threshold of 2.

The training process for all these models utilized the following hyperparameters:
\begin{itemize}
    \item Epochs: 5
    \item Batch Size: 20
    \item Optimizer: Adam
    \item Learning Rate: 0.001
    \item Loss Function: Cross-Entropy Loss (\texttt{nn.CrossEntropyLoss})
\end{itemize}

\subsection{Licenses}

\newlength{\licenselabelwidth}
\setlength{\licenselabelwidth}{22em} 

\begin{itemize}
    \item \makebox[\licenselabelwidth][l]{GPT2 model:} MIT license
    \item \makebox[\licenselabelwidth][l]{Qwen models:} Apache 2.0 license
    \item \makebox[\licenselabelwidth][l]{RWKV models:} Apache 2.0 license
    \item \makebox[\licenselabelwidth][l]{Mamba models:} Apache 2.0 license
    \item \makebox[\licenselabelwidth][l]{Pythia models:} Apache 2.0 license
    \item \makebox[\licenselabelwidth][l]{ResNet models:} Apache 2.0 license
    \item \makebox[\licenselabelwidth][l]{ESM2 models:} MIT license
    \item \makebox[\licenselabelwidth][l]{WikiText2 dataset:} CC BY-SA 3.0 license and GFDL
    \item \makebox[\licenselabelwidth][l]{WikiText103 dataset:} CC BY-SA 3.0 license and GFDL
    \item \makebox[\licenselabelwidth][l]{text8 dataset:} License not specified 
    \item \makebox[\licenselabelwidth][l]{1BW (One Billion Word) dataset:} Apache 2.0 license
    \item \makebox[\licenselabelwidth][l]{Pile dataset:} MIT license
    \item \makebox[\licenselabelwidth][l]{LAMBADA benchmark:} CC BY 4.0 license
    \item \makebox[\licenselabelwidth][l]{CBT (Children's Book Test) benchmark:} GFDL
    \item \makebox[\licenselabelwidth][l]{ARC benchmark:} CC BY-SA 4.0 license
    
\end{itemize}
\end{document}